%% file: Formatting-Instructions-LaTeX-2023.tex
\documentclass[letterpaper]{article} 
\usepackage{aaai23}  
\usepackage{times}  
\usepackage{helvet}  
\usepackage{courier}  
\usepackage[hyphens]{url}  
\usepackage{graphicx} 
\usepackage{natbib}  
\usepackage{caption} 
\usepackage{algorithm}
\usepackage{algorithmic}

\usepackage{newfloat}
\usepackage{listings}
\DeclareCaptionStyle{ruled}{labelfont=normalfont,labelsep=colon,strut=off} 
\floatstyle{ruled}
\newfloat{listing}{tb}{lst}{}
\floatname{listing}{Listing}
\usepackage{amsfonts}       
\usepackage{amsmath}
\usepackage{dsfont}
\usepackage{nicefrac}       
\usepackage{microtype}      
\usepackage{xcolor}         
\usepackage{tikz}
\usepackage{pgfplots}
\usetikzlibrary{automata, positioning, arrows, calc}

\usepackage{subfig}

\usepackage{amsthm}
\newtheorem{theorem}{Theorem}[section]
\newtheorem{corollary}[theorem]{Corollary}
\newtheorem{lemma}[theorem]{Lemma}

\input{commands}

\title{On the Sample Complexity of Vanilla Model-Based Offline Reinforcement Learning with Dependent Samples}
\author{
    Mustafa O. Karabag,
    Ufuk Topcu
}
\affiliations{
The University of Texas at Austin, Austin, TX, USA \\
    \{karabag, utopcu\}@utexas.edu
}

\begin{document}

\maketitle

\begin{abstract}
Offline reinforcement learning (offline RL) considers problems where learning is performed using only previously collected samples and is helpful for the settings in which collecting new data is costly or risky. In model-based offline RL, the learner performs estimation (or optimization) using a model constructed according to the empirical transition frequencies. We analyze the sample complexity of vanilla model-based offline RL with dependent samples in the infinite-horizon discounted-reward setting. In our setting, the samples obey the dynamics of the Markov decision process and, consequently, may have interdependencies. Under no assumption of independent samples, we provide a high-probability, polynomial sample complexity bound for vanilla model-based off-policy evaluation that requires partial or uniform coverage. We extend this result to the off-policy optimization under uniform coverage. As a comparison to the model-based approach, we analyze the sample complexity of off-policy evaluation with vanilla importance sampling in the infinite-horizon setting. Finally, we provide an estimator that outperforms the sample-mean estimator for almost deterministic dynamics that are prevalent in reinforcement learning.  
\end{abstract}

\input{introduction}

\input{preliminaries}

\input{settingproblemdef}

\input{offpolicyevaluation}
\input{beyondsamplemean}

\input{relatedwork}

\input{conclusion}

\section*{Acknowledgments}
This work was supported in part by ARO W911NF2110009 and NSF 1652113.

\bibliography{ref}

\appendix

\input{appendixnew2}

\newpage

\input{appendiximportance}
\newpage
\input{appendixdfestimator}

\end{document}

%% file: commands.tex
\newcommand{\mdp}{\mathcal{M}}
\newcommand{\states}{\mathcal{S}}
\newcommand{\actions}{\mathcal{A}}
\newcommand{\probs}{\mathcal{P}}
\newcommand{\cardstates}{S}
\newcommand{\cardactions}{A}
\newcommand{\rewards}{r}
\newcommand{\genericstate}{s}
\newcommand{\altstate}{q}
\newcommand{\genericaction}{a}
\newcommand{\genericreward}{r}
\newcommand{\initialstate}{\genericstate_{0}}
\newcommand{\deadstate}{\genericstate_{\circ}}
\newcommand{\policy}{\pi}
\newcommand{\stateactionrewardpath}{\xi}
\newcommand{\valuefunction}[1]{V^{#1}}
\newcommand{\target}{t}
\newcommand{\behavior}{b}
\newcommand{\behaviorpolicy}{\policy^{b}}
\newcommand{\targetpolicy}{\policy^{t}}
\newcommand{\optimalpolicy}{\policy^{*}}

\newcommand{\processcontinueprob}{\gamma}
\newcommand{\expectation}[1]{\mathbb{E}\left[ #1 \right]}
\newcommand{\occupancymeasure}{x}

\newcommand{\numberofpaths}{N}
\newcommand{\numberofsampletransitions}{n}
\newcommand{\reachprob}{\rho}
\newcommand{\loopprob}{\lambda}
\newcommand{\failureprob}{\delta}

\newcommand{\optimalitygap}{\varepsilon}
\newcommand{\pathdist}[1]{\Gamma^{#1}}
\newcommand{\apath}{\xi}

%% file: introduction.tex
\section{Introduction}

Offline reinforcement learning (RL) considers problems where a learner has access to only a dataset that is collected under a \textit{behavior policy} in an environment and tries to evaluate (or optimize) a \textit{target policy}. The learner typically has no control over the behavior policy, and the transition dynamics of the environment are unknown to the learner. Offline RL is helpful for settings where online learning may not be safe or previously collected data are abundant. The applications of offline RL include, but are not limited to healthcare~\cite{shortreed2011informing,tseng2017deep}, robotics~\cite{levine2018learning,ebert2018visual,zeng2018learning}, natural language processing~\cite{zhou2017end,henderson2008hybrid}, and recommendation systems~\cite{swaminathan2017off,gilotte2018offline}.

We develop theoretical guarantees for offline RL. In detail, we use an infinite horizon Markov decision process (MDP) to model the environment and analyze the number of sample paths sufficient to achieve the desired accuracy for off-policy evaluation. We mainly focus on vanilla model-based off-policy evaluation, where a target policy is evaluated through a model based on the sample mean estimator of the transition dynamics.  

Analyzing the theoretical properties of model-based off-policy evaluation is challenging due to the sequential nature of the MDP model and potentially dependent samples. These factors make model-based off-policy evaluation lack the unbiasedness property that importance-sampling-based off-policy methods have~\cite{levine2018learning}. 
The first source of bias is because the expected value is a non-linear function of the transition probabilities. 
Under the assumption that the transition probability estimates are unbiased, the bias in the value function estimate vanishes asymptotically with the increasing number of sample transitions. However, this bias is present with any finite number of samples~\cite{mannor2004bias}. 
The second source of bias is because of potentially biased transition probability estimates.
In reality, the sample transitions come from time series data and are not necessarily independent. The sample mean estimator, consequently, is not guaranteed to be unbiased. 
Quantifying this bias requires knowing the model, which contradicts the motivations of RL.

We consider that the dataset is constructed using sample paths that are executions of an MDP under the behavior policy and derive a sample complexity upper bound for model-based off-policy evaluation. 
We overcome the first source of bias by using a robust MDP~\cite{NE:05} that includes the true MDP with high probability. 
To overcome the second source of bias, we use a concentration bound that can handle random stopping times potentially dependent on the previous samples.
We combine these methods and derive a sufficient condition on the number of sample paths to be collected to accurately estimate the value of the target policy with high probability. 
The bound shows that the vanilla model-based off-policy evaluation has performance guarantees under both partial and uniform coverage.
We extend this sample complexity result to off-policy optimization under uniform coverage. In addition, as a comparison, we derive a sufficient condition on the number of samples for the vanilla importance sampling method. 
Finally, we give an estimator that outperforms the sample mean estimator in settings where transition dynamics of the MDP is almost deterministic, i.e., there is a probable next state for every state and action. 

The main contributions of this paper are threefold:
\begin{enumerate}
\item We derive sufficient conditions on the number of sample paths for vanilla model-based off-policy evaluation and optimization. These bounds do not assume independence between the sample transitions.
\item We derive a sufficient condition on the number of sample paths for the importance-sampling-based off-policy evaluation in the infinite-horizon discounted reward setting.
\item We provide a new estimator for the transition probabilities that outperforms the sample mean estimator for the environments with a limited amount of stochasticity.
\end{enumerate}
We remark that for the first two contributions, we aim to analyze the performance of vanilla off-policy methods in the discounted infinite horizon setting rather than building new algorithms with optimal sample complexities. 

The rest of the paper is organized as follows. Section \ref{sec:prelim} gives preliminaries for MDPs. We describe the off-policy evaluation and optimization problems in Section \ref{sec:offlinerl} and discuss the bias issues in model-based offline learning. In Section \ref{sec:theory}, we give the sample complexity results for the vanilla off-policy methods. We describe a new estimator in Section \ref{sec:dfestimator} for almost-deterministic random variables and analyze the performance of this estimator. We give the related work in Section \ref{sec:relatedwork}.

%% file: preliminaries.tex
\section{Preliminaries} \label{sec:prelim}
A Markov decision process (MDP) is a tuple \(\mdp = (\states, \actions, \probs, \rewards, \initialstate)\) where \(\states\) is the set of states, \(\actions\) is the set of actions, \(\probs(\genericstate, \genericaction, \altstate)\) is the transition probability form state \(\genericstate\) to \(\altstate\) under action \(\genericaction\), \(\rewards(\genericstate, \genericaction)\) is the (random) reward of action \(\genericaction\) at state \(\genericstate\), and \(\initialstate\) is the initial state. We assume that the reward is normalized, i.e., \(0 \leq \expectation{\rewards(s,a)} \leq 1\) for all \(\genericstate \in \states\), and \(\genericaction \in \actions\). \(\cardstates\) denotes the cardinality of \(\states\) and \(\cardactions\) denotes the cardinality of \(\actions\).
An absorbing state \(\genericstate\) transitions to itself under every action and has \(0\) reward, i.e., \(\probs(\genericstate, \genericaction, \genericstate)=1\) and \(\rewards(\genericstate, \genericaction) = 0\) for all \(\genericaction \in \actions\). 
A (stationary) policy \(\policy\) assigns the same probabilities to actions given the state at every time step; \(\policy(\genericstate, \genericaction)\) denotes the probability of taking action \(\genericaction\) at state \(\genericstate\). 
An (infinite) path \(\stateactionrewardpath = \genericstate_{0} \genericaction_{0} \genericreward_{0} \genericstate_{1} \genericaction_{1} \genericreward_{1} \ldots\) is a sequence of states, actions, and rewards. 
The value function \(\valuefunction{\policy}_{\mdp}(\genericstate)\) denotes the expected total reward under policy \(\policy\) starting from \(\genericstate\), i.e., \(\valuefunction{\policy}_{\mdp}(\genericstate) = \expectation{\sum_{t=0}^{\infty} \rewards(\genericstate_{t}, \genericaction_{t}) } \) where the expectation is over the randomness of the policy, transition dynamics, and rewards. 

The occupancy measure \(\occupancymeasure^{\policy}(\genericstate, \genericaction)\) denotes the expected number of times that action \(\genericaction\) is taken at state \(\genericstate\) under policy \(\policy\) starting from \(\initialstate\). 
Due to the linearity of expectation, we have \(\valuefunction{\policy}_{\mdp}(\initialstate) = \sum_{\genericstate \in \states,  \\ \genericaction \in \actions} \occupancymeasure^{\policy}(\genericstate, \genericaction) \expectation{ \rewards(\genericstate, \genericaction) } \). Define \(\reachprob^{\policy}(\genericstate, \genericaction)\) as the probability of taking action \(\genericaction\) at state \(\genericstate\) at least once under stationary policy \(\policy\) starting from \(\initialstate\). Also, define \(\loopprob^{\policy}(\genericstate, \genericaction)\) as the probability of taking action \(\genericaction\) at state \(\genericstate\) again under stationary policy \(\policy\) given that the current state is \(\genericstate\) and current action is \(\genericaction\). Due to the Markovianity of the transition dynamics and the stationarity of policy \(\policy\), we have
\begin{align*}
    \occupancymeasure^{\policy}(\genericstate, \genericaction)&= \reachprob^{\policy}(\genericstate, \genericaction) \sum_{i=1}^{\infty} (1-\loopprob^{\policy}(\genericstate, \genericaction)) \loopprob^{\policy}(\genericstate, \genericaction)^{i-1}i 
    \\
    &= \frac{\reachprob^{\policy}(\genericstate, \genericaction)}{1-\loopprob^{\policy}(\genericstate, \genericaction)}
\end{align*} where \(i\) represents the number of times \((\genericstate, \genericaction)\) is used.

%% file: settingproblemdef.tex
\section{Offline Reinforcement Learning Problem}
\label{sec:offlinerl}
We consider two offline reinforcement learning problems. 
The first problem is off-policy evaluation where the goal is to estimate the value \(\valuefunction{\targetpolicy}_{\mdp}(\initialstate)\) of a known stationary target policy \(\targetpolicy\) given \(\numberofpaths\) sample paths that are collected under a known stationary behavior policy \(\behaviorpolicy\). 
The second problem is off-policy optimization where the goal is to synthesize an optimal policy \(\policy^{*}\) that maximizes the value function \(\valuefunction{\optimalpolicy}_{\mdp}(\initialstate)\) given \(\numberofpaths\) sample paths that are collected under a known stationary  behavior policy \(\behaviorpolicy\). 

For both of these problems, we assume that there exists an absorbing final state \(\deadstate\) such that every state in \(\states \setminus \lbrace \deadstate \rbrace\) transitions to \(\deadstate\) with probability \(1 - \processcontinueprob\) under every action. 
State \(\deadstate\) represents the effective end of the path. 
We note that transitioning to \(\deadstate\) with a fixed probability is equivalent to having a discount factor \(\processcontinueprob\) and ensures the boundedness of the value function. 
On the other hand, the setting we consider is more disadvantaged compared to having infinite-length paths with discounted rewards  since sample paths eventually end up at \(\deadstate\) and the learner cannot access to further sample transitions from the other states.

%% file: offpolicyevaluation.tex
\subsection{Vanilla Model-Based Offline Learning}
We analyze the vanilla model-based approach for the aforementioned offline reinforcement learning problems. In this section, we describe the model construction.

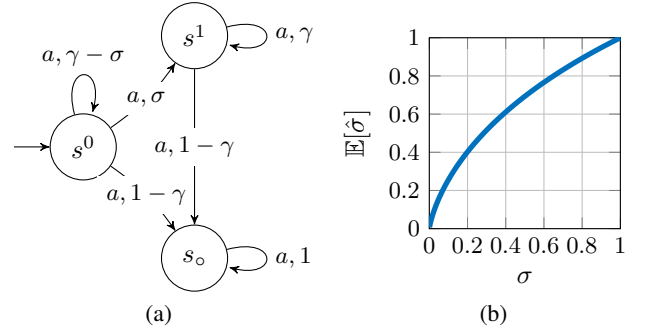
\begin{figure}
    \centering
    \subfloat[\centering ]{\label{fig:twostatemdp}\input{mdpfigure}}%
    \subfloat[\centering ]{\label{fig:biasgraph} \input{biasplot} }%
    \caption{(a) An MDP with a single action. A label $a,p$ of a directed edge from $\genericstate$ to $\altstate$ means $\probs(\genericstate, \genericaction, \altstate) = p$. (b) The expected sample mean estimation \(\expectation{\hat{\sigma}}\) of \(\sigma\) given a single path when \(\processcontinueprob = 1\).} 
\end{figure}

The model construction is fairly simple; we utilize the sample mean estimator to estimate the transition probabilities and rewards. Formally, let \(\numberofsampletransitions(\genericstate, \genericaction, \altstate)\) be the total number of sample transitions in the sample paths from state \(\genericstate\) to state \(\altstate\) under action \(\genericaction\). Let  \(\tilde{\probs}(\genericstate, \genericaction, \altstate) = \ \numberofsampletransitions(\genericstate, \genericaction, \altstate) / \left(  \sum_{\altstate \in \states} \numberofsampletransitions(\genericstate, \genericaction, \altstate) \right) \) denote the empirical frequency of transitioning from \(\genericstate\) to \(\altstate\) under \(\genericaction\). The estimated transition probabilities \(\hat{\probs}(\genericstate, \genericaction, \altstate)\) minimize the \(L_{2}\) distance to the empirical frequencies subject to the constraint \(\hat{\probs}(\genericstate, \genericaction, \deadstate) = 1- \processcontinueprob.\) If state \(\genericstate\) has no sample transition, we set \(\hat{\probs}(\genericstate, \genericaction, \genericstate) = \processcontinueprob \) for all \(\genericaction \in \actions\). For simplicity, we assume that the mean reward \(\expectation{\rewards(\genericstate, \genericaction)}\) is known for all \(\genericstate \in \states\) and \(\genericaction \in \actions\). The results we present in this paper can be extended to the unknown reward case by considering a sample mean estimator for the rewards as well.

Given the estimated model \(\hat{\mdp} = (\states, \actions, \hat{\probs}, \rewards, \initialstate)\) and the target policy \(\targetpolicy\), the value of the target policy can be estimated by solving a set of equations or by value iteration. We assume that this computation can be performed exactly. \(\valuefunction{\targetpolicy}_{\hat{\mdp}}(\initialstate)\) denotes the model-based value estimate for \(\targetpolicy\).

\subsection{Bias Issues in Model-Based Offline Learning}

The main challenges associated with model-based off-policy evaluation are due to the biases in the estimation of transition probabilities and the estimation of the value function.
\paragraph{Bias in the estimation of transition probabilities} We note that the sample-mean estimator is the maximum likelihood estimator of the transition probabilities. However, this estimator is biased since the outcomes of the future samples are dependent on the previous samples. For example, consider the MDP given in Figure \ref{fig:twostatemdp}. A transition from \(s^{0}\) to \(s^{1}\) implies that all previous transitions are from \(s^{0}\) to \(s^{0}\). The (potential) dependencies between the samples, i.e., the dependency between the number of samples and the outcomes of the sample transitions, make the sample mean estimate biased. As shown in Figure \ref{fig:biasgraph}, the bias of the sample mean estimator can be as large as \(0.22\) given a single sample path. In general, this bias occurs if the successor states of an origin state have different return probabilities to the origin state. 

While the model-based off-policy estimation is provably good when the estimates for the transition probabilities are unbiased~\cite{mannor2004bias}, it is challenging to obtain unbiased estimates with low variance since the learning data usually consists of samples that have dependencies. A provably unbiased estimator utilizes only the first sample from the origin state. However, this estimator suffers from high variance due to the low number of samples. Another option to overcome the bias issue is fixing a number of samples per state-action pair as in the PAC-MDP literature~\cite{strehl2008analysis}. However, this approach may result in a large number of ``unknown'' state-action pairs and is wasteful in that not all samples are utilized. 

We overcome the bias issue by considering a concentration bound that can work with a random number of samples and handle the dependencies between the outcomes.

\paragraph{Bias in the value function estimation} Even when the transition probability can be estimated without a bias, the estimate for the value function is, in general, biased. Given the estimates for the transition probabilities are unbiased, the bias and variance of the value function estimate vanish as the number of samples per state approach to \(\infty\)~\cite{mannor2004bias}. A way to overcome the bias in the value function estimation is to use a robust MDP model that uses a possible set of transition probabilities~\cite{yu2020mopo,yu2021combo}. The robust model is then used to compute upper and lower bounds on the value function. We also follow this approach and use a robust MDP to show that the value function estimate is accurate with high probability. 

 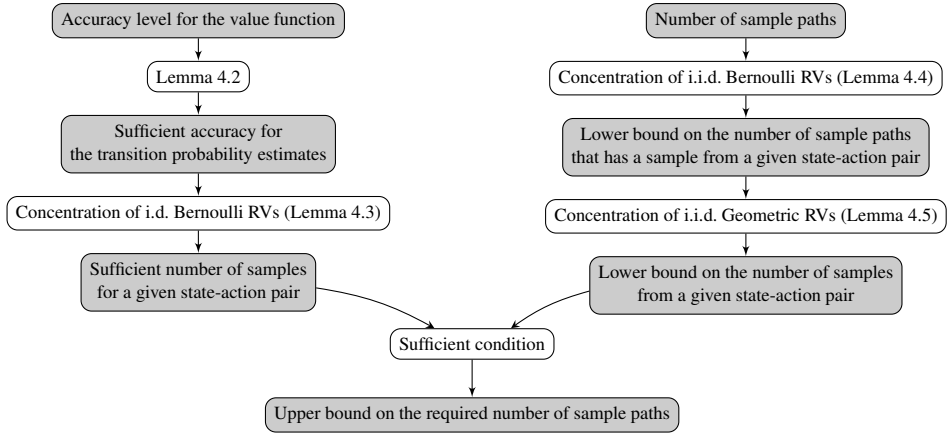
\begin{figure*}%
    \centering
    \input{frameworktikz} 
    \caption{The flowchart for the proof of Theorem \ref{thm:numberofrequiredpaths}. Gray boxes are the relevant quantities, and white boxes are the relations between these quantities. }
    \label{fig:flowcharttikz}
\end{figure*}

\section{Theoretical Guarantees for Vanilla Model-Based Off-Policy Evaluation and Optimization}
\label{sec:theory}

In this section, we analyze the performance of vanilla model-based off-policy evaluation and optimization. 

We derive a bound on the number \(\numberofpaths\) of required sample paths that relates the estimation accuracy to the distance between the behavior and target policies. The bound is polynomial in the statistics of the behavior and target policies, and the size of the MDP. 

\begin{theorem} \label{thm:numberofrequiredpaths}
Let \(\numberofpaths\) be the number of sample paths that are independently collected under \(\behaviorpolicy\). Define \[D = \left\lbrace (\genericstate, \genericaction) \bigg|  \occupancymeasure^{\targetpolicy}(\genericstate, \genericaction) \geq \frac{(\nicefrac{\optimalitygap}{2})^{\nicefrac{1}{\beta}} (1-\processcontinueprob)^{\nicefrac{(2-\beta)}{\beta}}}{\cardstates \cardactions} \right\rbrace.\]
If \begin{align*}
    \numberofpaths \geq \tilde{\mathcal{O}} \left(  \underset{\beta \in [0,1]}{\min} \ \underset{(\genericstate, \genericaction) \in D}{\max} \left(  \frac{\cardstates^{1+2\beta}  \cardactions^{2\beta} }{(1-\processcontinueprob)^{4-2\beta}}  \cdot \frac{ \occupancymeasure^{\targetpolicy}(\genericstate, \genericaction)^{2\beta}}{ \occupancymeasure^{\behaviorpolicy}(\genericstate, \genericaction)} \cdot \frac{1 }{\optimalitygap^{2} } , \right. \right. 
    \\
    \left. \left. \frac{1}{\processcontinueprob \reachprob^{\behaviorpolicy}(\genericstate, \genericaction)}\right)\right)
\end{align*}
then with probability at least \(1- \failureprob\), we have \[|\valuefunction{\targetpolicy}_{\mdp}(\initialstate) - \valuefunction{\targetpolicy}_{\hat{\mdp}}(\initialstate)| \leq \optimalitygap\] where the dependency on \(\nicefrac{1}{\failureprob}\) is logarithmic.
\end{theorem}

The proof is given in Appendix \ref{appendix} of the supplementary material.

The bound in Theorem \ref{thm:numberofrequiredpaths} holds for every \(\beta \in [0,1]\). This implies that vanilla model-based off-policy evaluation works under both uniform coverage and partial coverage. For \(\beta = 0\), the bound depends on how uniformly \(\behaviorpolicy\) covers the state-action space, i.e., \(\max \nicefrac{1}{ \occupancymeasure^{\behaviorpolicy}(\genericstate, \genericaction)}\). If \(\beta > 0\), the bound depends on the distributional shift between the policies; it is sufficient that \(\behaviorpolicy\) covers the state-action pairs that are frequently visited by \(\targetpolicy\). We note that we do not need to decide on the value of \(\beta\) a priori. We also note that the maximum is over the set \(D\) of state-action pairs for which the target policy has a sufficiently large occupancy measure. This implies that vanilla model-based estimation remains to be accurate for pathological cases where some parts of the MDP are unreachable or reached with a very low probability under both policies. 

The first fraction in the bound given in Theorem \ref{thm:numberofrequiredpaths} shows that as the occupancy measures of the behavior and target policies get close to each other in terms of ratio, then the off-policy estimation gets more accurate. In detail, the sufficient number of sample paths increase when the size of the MDP \(\cardstates\), the maximum occupancy measure \(\nicefrac{1}{1-\processcontinueprob}\), the desired accuracy \(\nicefrac{1}{\optimalitygap}\) or the distributional shift \(\nicefrac{ \occupancymeasure^{\targetpolicy}(\genericstate, \genericaction)}{ \occupancymeasure^{\behaviorpolicy}(\genericstate, \genericaction)}\) between the behavior and target policies increase. The last fraction in the bound is a natural consequence of rejection sampling: \(\tilde{\mathcal{O}}(\nicefrac{1}{\processcontinueprob \reachprob^{\behaviorpolicy}(\genericstate, \genericaction)})\) sample paths are required to ensure that there is at least one path that has a sample from \((\genericstate, \genericaction)\) to \(\states \setminus \lbrace \deadstate \rbrace\).

In the extreme case where the behavior and target policies are the same, the bound has \(\nicefrac{1}{\optimalitygap^2}\) dependence on the desired optimality gap. We note that the expected reward of a random path is subgaussian, and the \(\nicefrac{1}{\optimalitygap^2}\) dependency matches the Chernoff bound.

\paragraph{Proof sketch for Theorem \ref{thm:numberofrequiredpaths}} We follow the steps shown in Figure \ref{fig:flowcharttikz} to prove Theorem \ref{thm:numberofrequiredpaths}. We first decide on the sufficient accuracy level for the transition probability estimates for a given level of accuracy for the value function. Next, we decide on a sufficient number of samples from a given state-action pair to achieve the accuracy level for the transition probability estimates. Finally, we decide on a sufficient number of sample paths to collect the sufficient number of samples from given state-action pairs.

We first decide on the required accuracy for the transition probability estimates from each state to make an accurate estimation for the value function.
Lemma \ref{lemma:closemdpcloseestimate} shows that if the transition probabilities of state-action pairs are accurate proportionally to their occupancy measures, then the estimated value function is accurate. We note that this lemma is similar to simulation lemma~\cite{strehl2008analysis}; however, unlike the simulation lemma, we do not assume a fixed accuracy level for every state-action pair for the estimation of transition probabilities.

\begin{lemma} \label{lemma:closemdpcloseestimate}
For any \(\alpha \geq 0\) and \(0 \leq \beta \leq 1\), if \[\sum_{\altstate \in \states} |\hat{\probs}(\genericstate, \genericaction, \altstate) - \probs(\genericstate, \genericaction, \altstate)| \leq \frac{\alpha}{\occupancymeasure^{\targetpolicy}(\genericstate, \genericaction)^{\beta} }\] for all \(\genericstate \in \states\) and \(\genericaction \in \actions\), then \[|\valuefunction{\targetpolicy}_{\mdp}(\initialstate) - \valuefunction{\targetpolicy}_{\hat{\mdp}}(\initialstate)| \leq \frac{\alpha (\cardstates \cardactions)^{\beta}}{(1-\processcontinueprob)^{2-\beta}} .\]
\end{lemma}

We set \(\alpha = \nicefrac{\optimalitygap (1-\processcontinueprob)^{(2-\beta)}}{(\cardstates \cardactions)^{\beta}}\) in Lemma \ref{lemma:closemdpcloseestimate} to achieve \(\optimalitygap\) accuracy. 
 Given Lemma \ref{lemma:closemdpcloseestimate}, our goal is to determine the number of sample paths that guarantee a desired estimation accuracy. In order to do so, we first determine the number of sample transitions that is sufficient to estimate the transition probabilities accurately. Lemma \ref{lemma:numberofsamplesfromsaq} provides an upper bound on the number of samples from a state-action pair to estimate transition probabilities within a desired accuracy.

\begin{lemma} \label{lemma:numberofsamplesfromsaq}
 For any \(0 < \failureprob' < 1\),  if \[ \sum_{\altstate \in \states\setminus \lbrace \deadstate \rbrace }\numberofsampletransitions(\genericstate, \genericaction, \altstate) \geq \frac{40 \cardstates}{(\optimalitygap')^2} \log\left(\frac{1}{\optimalitygap'}\right) \log\left( \frac{5}{3 \failureprob} \right)  , \] then \( \sum_{\altstate \in \states}|\hat{\probs}(\genericstate, \genericaction, \altstate) - \probs(\genericstate, \genericaction, \altstate)| \leq \processcontinueprob \optimalitygap' \) with probability at least \(1 - \failureprob'\).
\end{lemma}
 
 To prove Lemma \ref{lemma:numberofsamplesfromsaq}, we use a concentration bound that can handle a random number of samples and possible dependencies between the samples. The bound is a random stopping time generalization of the i.i.d. concentration inequality given in \cite{weissman2003inequalities} for the categorical distributions. Thanks to this bound, we overcome the aforementioned bias problem in estimating transition probabilities during our analysis. By Lemma \ref{lemma:numberofsamplesfromsaq}, to achieve the accuracy given in Lemma \ref{lemma:closemdpcloseestimate}, \(\tilde{\mathcal{O}} \left( \nicefrac{\cardstates^{1+2\beta} \cardactions^{2\beta} \processcontinueprob^{2} {\occupancymeasure^{\targetpolicy}(\genericstate, \genericaction)}^{2\beta}}{\optimalitygap^{2}(1-\processcontinueprob)^{4-2\beta}} \right)\) samples from \((\genericstate, \genericaction)\) to \(\states \setminus \lbrace \deadstate \rbrace\) are sufficient.
 
 In the second part of the proof, we decide on the required number of paths to have enough samples from every state. We first decide on the number of sample paths that has a sample from \((\genericstate, \genericaction)\) to \(\states \setminus \lbrace \deadstate \rbrace\). Lemma \ref{lemma:numberofpathswithsample} provides a lower bound on the number of sample paths that has a sample from a given state-action pair.
\begin{lemma} \label{lemma:numberofpathswithsample}
Let \(\numberofpaths\) be the number of sample paths that are independently collected under \(\behaviorpolicy\). For any \(0 < \failureprob' < 1\), \(\numberofpaths' \geq 0\), and \((\genericstate, \genericaction) \in \states \times \actions\), if \[\numberofpaths \geq \frac{6}{\processcontinueprob \reachprob^{\behaviorpolicy}(\genericstate, \genericaction)} \max \left(  \numberofpaths',  \log(\nicefrac{1}{\failureprob'})\right),\] then the number of sample paths that has a sample from state \(\genericstate\) to \(\states \setminus \lbrace \deadstate \rbrace\) under action \(\genericaction\), is at least \(\numberofpaths'\) with probability at least \(1 - \failureprob'\).
\end{lemma} By Lemma \ref{lemma:numberofpathswithsample}, to have \(\numberofpaths'\) sample paths that has a sample from \((\genericstate, \genericaction)\) to \(\states \setminus \lbrace \deadstate \rbrace\), \(\tilde{\mathcal{O}} \left(\nicefrac{\numberofpaths'}{\processcontinueprob \reachprob^{\behaviorpolicy}(\genericstate,\genericaction)} \right)\) sample paths are sufficient. 

We note that if a path has a sample from \((\genericstate, \genericaction)\) to \(\states \setminus \lbrace \deadstate \rbrace\), then the number of samples from \((\genericstate, \genericaction)\) to \(\states \setminus \lbrace \deadstate \rbrace\) follows a geometric distribution due to the stationarity of \(\behaviorpolicy\). By a tail bound for the sum of geometric random variables~\cite{janson2018tail}, Lemma \ref{lemma:numberofsamplesperstateaction} provides an upper bound on the number of sample paths that has a sample from a given state-action pair, in order to ensure a desired number of sample transitions from the state-action pair. 
\begin{lemma} \label{lemma:numberofsamplesperstateaction}
Let \(\numberofpaths'\) be the number of sample paths that has a sample from state \(\genericstate\) to \(\states \setminus \lbrace \deadstate \rbrace\) under action \(\genericaction\). For any \(0 < \failureprob' < 1\) and \(k \geq 0\),  if \[ \numberofpaths' \geq \max \left( 8 k (1-\loopprob^{\behaviorpolicy}(\genericstate, \genericaction)), \log(\nicefrac{1}{\failureprob'}) \right) \] then the number of sample transitions from state \(\genericstate\) to \(\states \setminus \lbrace \deadstate \rbrace\) under action \(\genericaction\), is at least \(k\) with probability at least \(1 - \failureprob'\).
\end{lemma} By Lemma \ref{lemma:numberofsamplesperstateaction}, if \(\numberofpaths'\) paths has a sample from \((\genericstate, \genericaction)\) to \(\states \setminus \lbrace \deadstate \rbrace\), then we have \(\tilde{\mathcal{O}}(\nicefrac{\numberofpaths'}{(1- \loopprob^{\behaviorpolicy}(\genericstate, \genericaction))})\) sample transitions from \((\genericstate, \genericaction)\) to \(\states \setminus \lbrace \deadstate \rbrace\) with high probability. Combining these bounds and the sufficient number of samples per state-action pair, we conclude that if the number \(\numberofpaths\) of paths is greater than or equal to \begin{align*}
    \numberofpaths \geq \tilde{\mathcal{O}} \left(  \underset{\beta \in [0,1]}{\min} \ \underset{\substack{\genericstate \in \states \setminus \lbrace \deadstate \rbrace \\ \genericaction \in \actions}}{\max} \left(  \frac{\cardstates^{1+2\beta}  \cardactions^{2\beta} }{(1-\processcontinueprob)^{4-2\beta}}  \cdot \frac{ \occupancymeasure^{\targetpolicy}(\genericstate, \genericaction)^{2\beta}}{ \occupancymeasure^{\behaviorpolicy}(\genericstate, \genericaction)} \cdot \frac{1 }{\optimalitygap^{2} } , \right. \right. 
    \\
    \left. \left.  \frac{1}{\processcontinueprob \reachprob^{\behaviorpolicy}(\genericstate, \genericaction)}\right)\right), 
\end{align*} then the model-based off-policy estimate is \(\optimalitygap\)-accurate with high probability. Finally, we note that if \(\occupancymeasure^{\targetpolicy}(\genericstate, \genericaction)^{\beta} \leq \nicefrac{\optimalitygap (1-\processcontinueprob)^{1-\beta}}{2(\cardstates \cardactions)^{\beta}}\), then the transition probability estimates for \((\genericstate, \genericaction)\) trivially satisfies the condition given in Lemma \ref{lemma:closemdpcloseestimate}, and \(\beta\) can be arbitrarily chosen between \(0\) and \(1\). Hence, if \begin{align*}
    \numberofpaths \geq \tilde{\mathcal{O}} \left(  \underset{\beta \in [0,1]}{\min} \ \underset{(\genericstate, \genericaction) \in D}{\max} \left(  \frac{\cardstates^{1+2\beta}  \cardactions^{2\beta} }{(1-\processcontinueprob)^{4-2\beta}}  \cdot \frac{ \occupancymeasure^{\targetpolicy}(\genericstate, \genericaction)^{2\beta}}{ \occupancymeasure^{\behaviorpolicy}(\genericstate, \genericaction)} \cdot \frac{1 }{\optimalitygap^{2} } , \right. \right. 
    \\
    \left. \left.  \frac{1}{\processcontinueprob \reachprob^{\behaviorpolicy}(\genericstate, \genericaction)}\right)\right), 
\end{align*} then the model-based off-policy estimate is \(\optimalitygap\)-accurate with high probability where \[D = \left\lbrace (\genericstate, \genericaction) \bigg|  \occupancymeasure^{\targetpolicy}(\genericstate, \genericaction) \leq \frac{(\nicefrac{\optimalitygap}{2})^{\nicefrac{1}{\beta}} (1-\processcontinueprob)^{\nicefrac{(1-\beta)}{\beta}}}{\cardstates \cardactions} \right\rbrace.\]

\paragraph{Comparison with importance sampling} As a comparison for the bound given in Theorem \ref{thm:numberofrequiredpaths}, we derive a bound on the sufficient number of sample paths that need to be collected for the vanilla off-policy estimation with importance sampling in the infinite-horizon discounted-reward setting. 

The importance sampling estimate is \[\hat{V}_{\mdp}^{\targetpolicy}(\initialstate) = \frac{1}{\numberofpaths}\sum_{i=1}^{\numberofpaths} \frac{\Pr(\apath^{i}= \genericstate^{i}_{0} \genericaction^{i}_{0} \ldots | \targetpolicy) }{ \Pr(\apath^{i} = \genericstate^{i}_{0} \genericaction^{i}_{0} \ldots| \behaviorpolicy)} \sum_{t=0}^{\infty} \expectation{ \rewards(\genericstate^{i}_{t}, \genericaction^{i}_{t}) } \] where \(\apath^{1}, \ldots, \apath^{\numberofpaths}\) are sample paths collected under \(\behaviorpolicy\).
Let \(\pathdist{\target}\) and \(\pathdist{\behavior}\) denote the distribution of paths under the target and behavior policies, respectively. To accurately estimate the expected value of a given function with respect to the probability measure \(\pathdist{\target}\) with high probability using the sample paths that are collected under policy \(\pathdist{\behavior}\), approximately \[\exp\left(\mathbb{E}_{\apath\sim\pathdist{\targetpolicy}}\left[l(\apath)\right] + \mathcal{O}\left(\mathrm{Std}_{\apath\sim\pathdist{\targetpolicy}}\left(l(\apath)\right)\right)\right)\] samples are sufficient for importance sampling where \(l(\apath) = \log\left(\nicefrac{\Pr(\apath|\targetpolicy)}{\Pr(\apath|\behaviorpolicy)}\right)\)~\cite{chatterjee2018sample}. 
Considering that \(\targetpolicy\) and \(\behaviorpolicy\) are stationary, we have \[\mathbb{E}_{\apath\sim\pathdist{\targetpolicy}}\left[l(\apath)\right] \leq \frac{1}{1-\processcontinueprob} \log\left(\max_{\substack{\genericstate \in \states \setminus \deadstate \\ \genericaction \in \actions}} \frac{\targetpolicy(\genericstate, \genericaction)}{\behaviorpolicy(\genericstate, \genericaction)} \right)\] and \[\mathrm{Std}_{\apath\sim\pathdist{\targetpolicy}}(l(\apath)) \leq \frac{\sqrt{2}}{1-\processcontinueprob} \log\left(\max_{\substack{\genericstate \in \states \setminus \deadstate \\ \genericaction \in \actions}} \frac{\targetpolicy(\genericstate, \genericaction)}{\behaviorpolicy(\genericstate, \genericaction)}\right).  \] The details of this derivation are given in Appendix \ref{appendiximportance} of the supplementary material. As a result, approximately
\begin{align*}
    &\exp\left(\mathbb{E}_{\apath\sim\pathdist{\targetpolicy}}\left[l(\apath)\right] + \mathcal{O}\left(\mathrm{Std}_{\apath\sim\pathdist{\targetpolicy}}\left(l(\apath)\right)\right)\right)
    \\
    &\leq \left(\max_{\substack{\genericstate \in \states \setminus \deadstate \\ \genericaction \in \actions}} \frac{\targetpolicy(\genericstate, \genericaction)}{\behaviorpolicy(\genericstate, \genericaction)}\right)^{\mathcal{O}\left(\frac{1}{1-\processcontinueprob}\right)}.
\end{align*} sample paths are sufficient for the vanilla off-policy estimation with importance sampling in the infinite-horizon discounted-reward case.

A form of the maximum distributional shift between the policies,  \(\max_{(\genericstate, \genericaction) \in D}  \left(\nicefrac{ \occupancymeasure^{\targetpolicy}(\genericstate, \genericaction)^{2\beta}}{ \occupancymeasure^{\behaviorpolicy}(\genericstate, \genericaction)}\right)\) for the model-based method and \(\max_{\substack{\genericstate \in \states \setminus \deadstate \\ \genericaction \in \actions}}  \left(\nicefrac{ \targetpolicy(\genericstate, \genericaction)}{ \behaviorpolicy(\genericstate, \genericaction)}\right)\) for the importance sampling method, appears in the upper bounds for both model-based off-policy estimation and off-policy estimation via importance sampling: As the inherit distance between the target and behavior policies increase, the problem of off-policy estimation becomes more challenging. 

We also note that the upper bound for the importance sampling has an exponential dependency on the expected time horizon, i.e.,  \(\nicefrac{1}{1-\processcontinueprob}\). Similar to the finite-horizon case, the variance of the estimates can grow exponentially with the expected time horizon in the infinite-horizon discounted-reward case.

\paragraph{Off-policy optimization} The estimated model can be used to maximize a known reward function. By letting \(\beta = 0\) in Theorem \ref{thm:numberofrequiredpaths}, we have the following result, which provides a bound on the number of paths that need to be collected for off-policy optimization.

\begin{corollary} \label{corollary:numberofrequiredpathsoptimization}
Let \(\numberofpaths\) be the number of sample paths that are independently collected under \(\behaviorpolicy\). Also, let \(\policy'\) denote optimal policy for the estimated model \(\hat{\mdp}\), and \(\policy^{*}\) denote optimal policy for the true model \(\mdp\). 
If the number \(\numberofpaths\) of paths is greater than or equal to \[ \tilde{\mathcal{O}} \left( \underset{\substack{\genericstate \in \states \setminus \lbrace \deadstate \rbrace \\ \genericaction \in \actions}}{\max}  \left(\frac{\cardstates \processcontinueprob  }{(1-\processcontinueprob)^4}  \cdot \frac{ 1}{ \occupancymeasure^{\behaviorpolicy}(\genericstate, \genericaction)} \cdot \frac{1 }{\optimalitygap^{2} } , \frac{1}{\processcontinueprob \reachprob^{\behaviorpolicy}(\genericstate, \genericaction)}\right)  \right),\] then with probability at least \(1- \failureprob\), we have \[|\valuefunction{\policy'}_{\mdp}(\initialstate) - \valuefunction{\policy^{*}}_{\mdp}(\initialstate)| \leq \optimalitygap,\]  where the dependency on \(\nicefrac{1}{\failureprob}\) is logarithmic.
\end{corollary}
The proof is given in Appendix \ref{appendix} of the supplementary material.

Corollary \ref{corollary:numberofrequiredpathsoptimization} shows the accuracy of vanilla model-based optimization under uniform coverage, i.e., dependency on \(\nicefrac{ 1}{ \occupancymeasure^{\behaviorpolicy}(\genericstate, \genericaction)}\). Related works such as \cite{yan2022efficacy,rashidinejad2021bridging} provide sample complexity bounds that require only partial coverage, i.e., \(\nicefrac{ \occupancymeasure^{\policy^{*}}(\genericstate, \genericaction)}{ \occupancymeasure^{\behaviorpolicy}(\genericstate, \genericaction)}\). The difference is because the vanilla model-based estimation is performed using the sample-mean estimates whereas the aforementioned works use a pessimistic MDP (as in the proof of Theorem \ref{thm:numberofrequiredpaths}). We can recover the same sample complexity bound given in Theorem \ref{thm:numberofrequiredpaths} for off-policy optimization with a pessimistic MDP built using the confidence bounds given in Lemma \ref{lemma:numberofsamplesfromsaq}.

\paragraph{Estimation of the Accuracy Level}
The learner is agnostic to the value of the bound given in Theorem \ref{thm:numberofrequiredpaths}. Computing the lower bound requires knowing the model fully, which inherently contradicts the motivation of off-policy learning. Consequently, the learner cannot directly know the accuracy level for a given confidence level. 

To estimate the level of accuracy, we can use the occupancy measure values that are computed using the estimated model. This approach gives an asymptotically consistent estimator of the accuracy level. Since the estimated model becomes more accurate with the increasing number of samples, the estimated occupancy measures become more accurate and, consequently, the estimate for the level of accuracy becomes more accurate.

In order to estimate the accuracy level with non-asymptotical guarantees, we need to compute the occupancy measures of the behavior and target policies. Computing the occupancy measures of the behavior policy is relatively easy and does not require the model construction: The sample paths are direct samples for the occupancy measures. Since the number of samples from a state-action pair in a path is a subexponential random variable, we can use Bernstein's inequality to compute a lower bound on the occupancy measures of the behavior policy. On the other hand, computing the occupancy measures of the target policy is challenging since the distribution of interest and the source of samples are not the same. For this computation, we can use robust MDPs. Lemma \ref{lemma:numberofsamplesfromsaq} gives a possible set of transition probabilities for a given confidence level. Using this lemma, we can build a robust MDP that contains the true transition probabilities with high probability. Then, we compute the largest possible occupancy measures for the target policy. Finally, given the robust estimates for the occupancy measures and the number of sample paths, we can compute a provably correct accuracy level using the bound given in Theorem \ref{thm:numberofrequiredpaths}.

%% file: mdpfigure.tex
		\begin{tikzpicture}[node distance = 1.2cm]
		    \tikzset{
	->, 
	>=stealth', 
	initial text=$ $, 
}
		
		    \node[state, initial] (initialstate) {$s^{0}$};
		    \node[state][above right= of initialstate] (secondstate) {$s^{1}$};
		    \node[state][below right= of initialstate] (deadstate) {$\deadstate$};	
			\draw 			
			(initialstate) edge[bend right = 10] node[fill=white] {\footnotesize $a, \sigma $} (secondstate)
			(initialstate) edge[loop above] node {\footnotesize $a, \processcontinueprob-\sigma$} (initialstate)
			(initialstate) edge[bend left = 10] node[fill=white] {\footnotesize $a, 1-\processcontinueprob$} (deadstate)
			(secondstate) edge node[fill=white] {\footnotesize $a, 1-\gamma$} (deadstate)
			(secondstate) edge[loop right] node {\footnotesize $a, \gamma$} (secondstate)
			(deadstate) edge[loop right] node {\footnotesize $a, 1$} (deadstate);
		\end{tikzpicture}

%% file: biasplot.tex
%
%
\definecolor{mycolor1}{rgb}{0.00000,0.44700,0.74100}%
\begin{tikzpicture}

\begin{axis}[%
width=1in,
height=1in,
at={(0in,0in)},
scale only axis,
unbounded coords=jump,
xmin=0,
xmax=1,
ticklabel style={font=\small},
label style={font=\small},
xlabel style={font=\color{white!15!black}},
xlabel={\(\sigma\)},
ymin=0,
ymax=1,
ylabel style={font=\color{white!15!black}},
ylabel={\(\mathbb{E}[\hat{\sigma}]\)},
axis background/.style={fill=white},
xmajorgrids,
ymajorgrids,
ylabel near ticks,
xlabel near ticks,
]
\addplot [color=mycolor1, line width=2pt, forget plot]
  table[row sep=crcr]{%
0	nan\\
0.001	0.00691466994893107\\
0.002	0.0124541244457359\\
0.003	0.0174798685766721\\
0.004	0.0221745418388042\\
0.005	0.0266247103846635\\
0.006	0.0308812624331232\\
0.007	0.0349777602311055\\
0.008	0.0389380140105024\\
0.009	0.0427797944649982\\
0.01	0.0465168705655363\\
0.011	0.0501602225156941\\
0.012	0.0537188092614672\\
0.013	0.0572000779936857\\
0.014	0.0606103157110916\\
0.015	0.0639548996631461\\
0.016	0.0672384805974367\\
0.017	0.0704651199325948\\
0.018	0.073638394480191\\
0.019	0.0767614777742082\\
0.02	0.0798372041924112\\
0.021	0.0828681201904321\\
0.022	0.0858565257297792\\
0.023	0.0888045081373784\\
0.024	0.0917139700483818\\
0.025	0.0945866526695881\\
0.026	0.0974241553028512\\
0.027	0.100227951850361\\
0.028	0.102999404862751\\
0.029	0.105739777570416\\
0.03	0.10845024424701\\
0.031	0.111131899184114\\
0.032	0.113785764501898\\
0.033	0.116412796978301\\
0.034	0.119013894046007\\
0.035	0.121589899080047\\
0.036	0.124141606077735\\
0.037	0.126669763815683\\
0.038	0.129175079554801\\
0.039	0.131658222353015\\
0.04	0.134119826036175\\
0.041	0.136560491870064\\
0.042	0.138980790970092\\
0.043	0.141381266480051\\
0.044	0.143762435546873\\
0.045	0.146124791114693\\
0.046	0.148468803558365\\
0.047	0.150794922173964\\
0.048	0.153103576541559\\
0.049	0.155395177773611\\
0.05	0.157670119660736\\
0.051	0.15992877972513\\
0.052	0.162171520190756\\
0.053	0.164398688878343\\
0.054	0.166610620032309\\
0.055	0.168807635085959\\
0.056	0.170990043370584\\
0.057	0.173158142773502\\
0.058	0.175312220349545\\
0.059	0.17745255289002\\
0.06	0.179579407452768\\
0.061	0.18169304185659\\
0.062	0.18379370514296\\
0.063	0.185881638007693\\
0.064	0.187957073204955\\
0.065	0.190020235925799\\
0.066	0.192071344153187\\
0.067	0.194110608995293\\
0.068	0.196138234998723\\
0.069	0.198154420443133\\
0.07	0.200159357618596\\
0.071	0.202153233086973\\
0.072	0.204136227928403\\
0.073	0.206108517973963\\
0.074	0.208070274025453\\
0.075	0.210021662063175\\
0.076	0.211962843442512\\
0.077	0.213893975080055\\
0.078	0.215815209629955\\
0.079	0.217726695651134\\
0.08	0.219628577765935\\
0.081	0.221520996810757\\
0.082	0.223404089979157\\
0.083	0.225277990957906\\
0.084	0.227142830056399\\
0.085	0.228998734329841\\
0.086	0.230845827696567\\
0.087	0.232684231049833\\
0.088	0.234514062364414\\
0.089	0.236335436798298\\
0.09	0.238148466789746\\
0.091	0.23995326214999\\
0.092	0.241749930151812\\
0.093	0.243538575614207\\
0.094	0.245319300983378\\
0.095	0.247092206410223\\
0.096	0.248857389824528\\
0.097	0.250614947006023\\
0.098	0.252364971652476\\
0.099	0.254107555444958\\
0.1	0.25584278811045\\
0.101	0.257570757481901\\
0.102	0.259291549555882\\
0.103	0.261005248547946\\
0.104	0.262711936945803\\
0.105	0.26441169556043\\
0.106	0.266104603575198\\
0.107	0.26779073859313\\
0.108	0.269470176682349\\
0.109	0.271142992419839\\
0.11	0.272809258933561\\
0.111	0.274469047943024\\
0.112	0.276122429798366\\
0.113	0.277769473518024\\
0.114	0.279410246825046\\
0.115	0.281044816182115\\
0.116	0.282673246825328\\
0.117	0.284295602796798\\
0.118	0.285911946976127\\
0.119	0.287522341110786\\
0.12	0.289126845845467\\
0.121	0.290725520750436\\
0.122	0.292318424348933\\
0.123	0.293905614143667\\
0.124	0.295487146642421\\
0.125	0.297063077382834\\
0.126	0.298633460956365\\
0.127	0.300198351031489\\
0.128	0.30175780037615\\
0.129	0.303311860879495\\
0.13	0.304860583572933\\
0.131	0.306404018650528\\
0.132	0.307942215488757\\
0.133	0.309475222665664\\
0.134	0.311003087979428\\
0.135	0.312525858466359\\
0.136	0.314043580418365\\
0.137	0.315556299399884\\
0.138	0.317064060264316\\
0.139	0.318566907169978\\
0.14	0.320064883595577\\
0.141	0.32155803235525\\
0.142	0.323046395613152\\
0.143	0.324530014897642\\
0.144	0.326008931115051\\
0.145	0.327483184563072\\
0.146	0.328952814943764\\
0.147	0.330417861376202\\
0.148	0.331878362408764\\
0.149	0.333334356031099\\
0.15	0.334785879685744\\
0.151	0.336232970279445\\
0.152	0.337675664194164\\
0.153	0.339113997297786\\
0.154	0.340548004954551\\
0.155	0.341977722035205\\
0.156	0.343403182926879\\
0.157	0.344824421542718\\
0.158	0.346241471331255\\
0.159	0.347654365285545\\
0.16	0.349063135952059\\
0.161	0.350467815439363\\
0.162	0.351868435426563\\
0.163	0.353265027171549\\
0.164	0.354657621519021\\
0.165	0.356046248908332\\
0.166	0.357430939381113\\
0.167	0.358811722588738\\
0.168	0.360188627799581\\
0.169	0.361561683906115\\
0.17	0.36293091943183\\
0.171	0.364296362537986\\
0.172	0.36565804103021\\
0.173	0.367015982364926\\
0.174	0.368370213655646\\
0.175	0.369720761679102\\
0.176	0.37106765288124\\
0.177	0.372410913383075\\
0.178	0.37375056898641\\
0.179	0.375086645179424\\
0.18	0.37641916714213\\
0.181	0.377748159751712\\
0.182	0.379073647587741\\
0.183	0.380395654937272\\
0.184	0.381714205799828\\
0.185	0.383029323892276\\
0.186	0.38434103265359\\
0.187	0.38564935524951\\
0.188	0.386954314577107\\
0.189	0.388255933269237\\
0.19	0.389554233698906\\
0.191	0.390849237983538\\
0.192	0.392140967989154\\
0.193	0.39342944533446\\
0.194	0.39471469139485\\
0.195	0.395996727306318\\
0.196	0.3972755739693\\
0.197	0.398551252052425\\
0.198	0.399823781996191\\
0.199	0.401093184016568\\
0.2	0.402359478108525\\
0.201	0.403622684049481\\
0.202	0.404882821402692\\
0.203	0.406139909520566\\
0.204	0.407393967547909\\
0.205	0.408645014425113\\
0.206	0.409893068891268\\
0.207	0.411138149487227\\
0.208	0.412380274558599\\
0.209	0.413619462258687\\
0.21	0.414855730551367\\
0.211	0.416089097213917\\
0.212	0.41731957983978\\
0.213	0.41854719584128\\
0.214	0.419771962452291\\
0.215	0.420993896730843\\
0.216	0.422213015561686\\
0.217	0.423429335658805\\
0.218	0.424642873567883\\
0.219	0.425853645668725\\
0.22	0.427061668177629\\
0.221	0.428266957149716\\
0.222	0.429469528481219\\
0.223	0.430669397911725\\
0.224	0.431866581026379\\
0.225	0.433061093258047\\
0.226	0.43425294988944\\
0.227	0.435442166055196\\
0.228	0.436628756743931\\
0.229	0.437812736800244\\
0.23	0.438994120926697\\
0.231	0.440172923685747\\
0.232	0.441349159501656\\
0.233	0.442522842662359\\
0.234	0.443693987321303\\
0.235	0.444862607499248\\
0.236	0.446028717086048\\
0.237	0.447192329842386\\
0.238	0.448353459401493\\
0.239	0.449512119270828\\
0.24	0.45066832283373\\
0.241	0.45182208335105\\
0.242	0.452973413962743\\
0.243	0.45412232768944\\
0.244	0.455268837433995\\
0.245	0.456412955983001\\
0.246	0.457554696008282\\
0.247	0.458694070068363\\
0.248	0.459831090609913\\
0.249	0.460965769969161\\
0.25	0.462098120373297\\
0.251	0.463228153941841\\
0.252	0.464355882687994\\
0.253	0.465481318519971\\
0.254	0.466604473242298\\
0.255	0.467725358557109\\
0.256	0.468843986065402\\
0.257	0.469960367268288\\
0.258	0.471074513568213\\
0.259	0.472186436270166\\
0.26	0.473296146582863\\
0.261	0.474403655619912\\
0.262	0.475508974400966\\
0.263	0.476612113852849\\
0.264	0.477713084810675\\
0.265	0.478811898018935\\
0.266	0.479908564132584\\
0.267	0.481003093718095\\
0.268	0.482095497254513\\
0.269	0.483185785134474\\
0.27	0.484273967665227\\
0.271	0.48536005506963\\
0.272	0.486444057487129\\
0.273	0.487525984974731\\
0.274	0.488605847507954\\
0.275	0.489683654981766\\
0.276	0.490759417211514\\
0.277	0.49183314393383\\
0.278	0.49290484480753\\
0.279	0.493974529414501\\
0.28	0.495042207260567\\
0.281	0.496107887776353\\
0.282	0.497171580318126\\
0.283	0.498233294168631\\
0.284	0.499293038537911\\
0.285	0.500350822564117\\
0.286	0.501406655314305\\
0.287	0.50246054578522\\
0.288	0.503512502904077\\
0.289	0.504562535529317\\
0.29	0.505610652451365\\
0.291	0.506656862393367\\
0.292	0.507701174011929\\
0.293	0.508743595897829\\
0.294	0.509784136576736\\
0.295	0.510822804509906\\
0.296	0.511859608094873\\
0.297	0.512894555666135\\
0.298	0.513927655495821\\
0.299	0.514958915794356\\
0.3	0.515988344711116\\
0.301	0.517015950335067\\
0.302	0.518041740695408\\
0.303	0.519065723762193\\
0.304	0.520087907446952\\
0.305	0.521108299603297\\
0.306	0.522126908027529\\
0.307	0.52314374045923\\
0.308	0.524158804581846\\
0.309	0.525172108023266\\
0.31	0.526183658356396\\
0.311	0.527193463099714\\
0.312	0.528201529717832\\
0.313	0.529207865622039\\
0.314	0.530212478170845\\
0.315	0.531215374670509\\
0.316	0.532216562375572\\
0.317	0.53321604848937\\
0.318	0.53421384016455\\
0.319	0.535209944503578\\
0.32	0.536204368559231\\
0.321	0.537197119335097\\
0.322	0.53818820378606\\
0.323	0.539177628818777\\
0.324	0.540165401292156\\
0.325	0.541151528017822\\
0.326	0.542136015760579\\
0.327	0.543118871238866\\
0.328	0.544100101125211\\
0.329	0.54507971204667\\
0.33	0.546057710585271\\
0.331	0.547034103278445\\
0.332	0.548008896619454\\
0.333	0.548982097057817\\
0.334	0.549953710999724\\
0.335	0.550923744808449\\
0.336	0.551892204804761\\
0.337	0.552859097267322\\
0.338	0.553824428433085\\
0.339	0.554788204497688\\
0.34	0.555750431615843\\
0.341	0.556711115901714\\
0.342	0.557670263429301\\
0.343	0.558627880232809\\
0.344	0.559583972307022\\
0.345	0.560538545607661\\
0.346	0.561491606051754\\
0.347	0.562443159517981\\
0.348	0.563393211847035\\
0.349	0.564341768841964\\
0.35	0.565288836268519\\
0.351	0.566234419855487\\
0.352	0.567178525295034\\
0.353	0.56812115824303\\
0.354	0.569062324319381\\
0.355	0.57000202910835\\
0.356	0.570940278158879\\
0.357	0.571877076984902\\
0.358	0.572812431065661\\
0.359	0.573746345846014\\
0.36	0.57467882673674\\
0.361	0.575609879114838\\
0.362	0.576539508323831\\
0.363	0.577467719674058\\
0.364	0.578394518442963\\
0.365	0.579319909875387\\
0.366	0.580243899183852\\
0.367	0.581166491548842\\
0.368	0.582087692119082\\
0.369	0.583007506011813\\
0.37	0.583925938313065\\
0.371	0.584842994077924\\
0.372	0.585758678330803\\
0.373	0.5866729960657\\
0.374	0.587585952246461\\
0.375	0.588497551807036\\
0.376	0.589407799651736\\
0.377	0.590316700655483\\
0.378	0.591224259664057\\
0.379	0.592130481494348\\
0.38	0.593035370934594\\
0.381	0.593938932744622\\
0.382	0.594841171656092\\
0.383	0.595742092372725\\
0.384	0.596641699570541\\
0.385	0.597539997898091\\
0.386	0.598436991976678\\
0.387	0.59933268640059\\
0.388	0.600227085737319\\
0.389	0.601120194527784\\
0.39	0.602012017286547\\
0.391	0.602902558502032\\
0.392	0.603791822636737\\
0.393	0.604679814127447\\
0.394	0.605566537385439\\
0.395	0.606451996796696\\
0.396	0.607336196722106\\
0.397	0.608219141497668\\
0.398	0.609100835434692\\
0.399	0.609981282819995\\
0.4	0.610860487916103\\
0.401	0.61173845496144\\
0.402	0.612615188170523\\
0.403	0.613490691734151\\
0.404	0.614364969819595\\
0.405	0.615238026570785\\
0.406	0.616109866108492\\
0.407	0.616980492530512\\
0.408	0.617849909911848\\
0.409	0.618718122304887\\
0.41	0.619585133739578\\
0.411	0.620450948223609\\
0.412	0.621315569742578\\
0.413	0.622179002260164\\
0.414	0.623041249718301\\
0.415	0.623902316037346\\
0.416	0.624762205116239\\
0.417	0.625620920832678\\
0.418	0.626478467043273\\
0.419	0.627334847583717\\
0.42	0.628190066268937\\
0.421	0.629044126893259\\
0.422	0.629897033230562\\
0.423	0.630748789034434\\
0.424	0.631599398038326\\
0.425	0.632448863955706\\
0.426	0.633297190480206\\
0.427	0.634144381285776\\
0.428	0.634990440026828\\
0.429	0.635835370338384\\
0.43	0.636679175836224\\
0.431	0.637521860117022\\
0.432	0.638363426758499\\
0.433	0.639203879319554\\
0.434	0.640043221340409\\
0.435	0.640881456342747\\
0.436	0.641718587829845\\
0.437	0.642554619286716\\
0.438	0.643389554180236\\
0.439	0.644223395959283\\
0.44	0.645056148054867\\
0.441	0.645887813880257\\
0.442	0.646718396831115\\
0.443	0.647547900285622\\
0.444	0.648376327604606\\
0.445	0.649203682131665\\
0.446	0.650029967193293\\
0.447	0.650855186099005\\
0.448	0.651679342141457\\
0.449	0.652502438596565\\
0.45	0.653324478723631\\
0.451	0.654145465765457\\
0.452	0.654965402948463\\
0.453	0.655784293482804\\
0.454	0.656602140562488\\
0.455	0.657418947365486\\
0.456	0.65823471705385\\
0.457	0.659049452773822\\
0.458	0.659863157655945\\
0.459	0.660675834815177\\
0.46	0.661487487350997\\
0.461	0.662298118347514\\
0.462	0.663107730873576\\
0.463	0.663916327982874\\
0.464	0.664723912714047\\
0.465	0.665530488090791\\
0.466	0.666336057121957\\
0.467	0.667140622801657\\
0.468	0.667944188109364\\
0.469	0.668746756010014\\
0.47	0.669548329454104\\
0.471	0.670348911377796\\
0.472	0.671148504703005\\
0.473	0.671947112337508\\
0.474	0.672744737175032\\
0.475	0.673541382095354\\
0.476	0.674337049964392\\
0.477	0.675131743634303\\
0.478	0.675925465943574\\
0.479	0.676718219717112\\
0.48	0.677510007766339\\
0.481	0.678300832889281\\
0.482	0.679090697870658\\
0.483	0.679879605481972\\
0.484	0.680667558481598\\
0.485	0.681454559614868\\
0.486	0.682240611614161\\
0.487	0.683025717198985\\
0.488	0.683809879076068\\
0.489	0.684593099939438\\
0.49	0.685375382470505\\
0.491	0.686156729338153\\
0.492	0.686937143198811\\
0.493	0.687716626696543\\
0.494	0.688495182463126\\
0.495	0.689272813118131\\
0.496	0.690049521269\\
0.497	0.690825309511129\\
0.498	0.691600180427942\\
0.499	0.692374136590975\\
0.5	0.693147180559945\\
0.501	0.693919314882832\\
0.502	0.694690542095953\\
0.503	0.695460864724036\\
0.504	0.696230285280297\\
0.505	0.69699880626651\\
0.506	0.697766430173085\\
0.507	0.698533159479135\\
0.508	0.699298996652554\\
0.509	0.700063944150085\\
0.51	0.700828004417389\\
0.511	0.701591179889119\\
0.512	0.702353472988988\\
0.513	0.703114886129839\\
0.514	0.703875421713711\\
0.515	0.70463508213191\\
0.516	0.705393869765075\\
0.517	0.706151786983244\\
0.518	0.706908836145923\\
0.519	0.707665019602148\\
0.52	0.708420339690553\\
0.521	0.709174798739435\\
0.522	0.709928399066816\\
0.523	0.710681142980509\\
0.524	0.71143303277818\\
0.525	0.712184070747409\\
0.526	0.712934259165757\\
0.527	0.713683600300821\\
0.528	0.714432096410302\\
0.529	0.71517974974206\\
0.53	0.715926562534179\\
0.531	0.716672537015022\\
0.532	0.717417675403295\\
0.533	0.718161979908103\\
0.534	0.718905452729007\\
0.535	0.719648096056086\\
0.536	0.720389912069991\\
0.537	0.721130902942003\\
0.538	0.72187107083409\\
0.539	0.722610417898964\\
0.54	0.723348946280133\\
0.541	0.72408665811196\\
0.542	0.724823555519716\\
0.543	0.725559640619636\\
0.544	0.726294915518969\\
0.545	0.727029382316037\\
0.546	0.727763043100283\\
0.547	0.72849589995233\\
0.548	0.729227954944023\\
0.549	0.729959210138494\\
0.55	0.730689667590203\\
0.551	0.731419329344992\\
0.552	0.732148197440141\\
0.553	0.732876273904408\\
0.554	0.733603560758091\\
0.555	0.734330060013067\\
0.556	0.735055773672848\\
0.557	0.735780703732625\\
0.558	0.736504852179323\\
0.559	0.737228220991642\\
0.56	0.737950812140108\\
0.561	0.738672627587121\\
0.562	0.739393669287002\\
0.563	0.740113939186035\\
0.564	0.740833439222523\\
0.565	0.741552171326823\\
0.566	0.742270137421401\\
0.567	0.742987339420869\\
0.568	0.743703779232037\\
0.569	0.744419458753953\\
0.57	0.74513437987795\\
0.571	0.745848544487689\\
0.572	0.746561954459201\\
0.573	0.747274611660934\\
0.574	0.747986517953792\\
0.575	0.748697675191182\\
0.576	0.749408085219051\\
0.577	0.750117749875933\\
0.578	0.750826670992988\\
0.579	0.751534850394045\\
0.58	0.752242289895642\\
0.581	0.752948991307067\\
0.582	0.753654956430399\\
0.583	0.754360187060549\\
0.584	0.755064684985297\\
0.585	0.755768451985335\\
0.586	0.756471489834307\\
0.587	0.757173800298842\\
0.588	0.757875385138602\\
0.589	0.758576246106312\\
0.59	0.759276384947803\\
0.591	0.759975803402051\\
0.592	0.76067450320121\\
0.593	0.761372486070654\\
0.594	0.762069753729011\\
0.595	0.762766307888202\\
0.596	0.763462150253478\\
0.597	0.764157282523455\\
0.598	0.76485170639015\\
0.599	0.765545423539017\\
0.6	0.766238435648986\\
0.601	0.766930744392493\\
0.602	0.767622351435518\\
0.603	0.768313258437622\\
0.604	0.769003467051976\\
0.605	0.769692978925402\\
0.606	0.770381795698402\\
0.607	0.771069919005195\\
0.608	0.771757350473749\\
0.609	0.772444091725816\\
0.61	0.773130144376964\\
0.611	0.773815510036609\\
0.612	0.774500190308052\\
0.613	0.775184186788508\\
0.614	0.775867501069136\\
0.615	0.77655013473508\\
0.616	0.77723208936549\\
0.617	0.777913366533562\\
0.618	0.778593967806565\\
0.619	0.779273894745874\\
0.62	0.779953148907\\
0.621	0.780631731839622\\
0.622	0.781309645087618\\
0.623	0.781986890189093\\
0.624	0.782663468676411\\
0.625	0.783339382076226\\
0.626	0.784014631909508\\
0.627	0.784689219691578\\
0.628	0.785363146932133\\
0.629	0.786036415135275\\
0.63	0.786709025799546\\
0.631	0.787380980417948\\
0.632	0.78805228047798\\
0.633	0.78872292746166\\
0.634	0.789392922845557\\
0.635	0.790062268100816\\
0.636	0.790730964693191\\
0.637	0.791399014083067\\
0.638	0.79206641772549\\
0.639	0.792733177070196\\
0.64	0.793399293561635\\
0.641	0.794064768638998\\
0.642	0.794729603736248\\
0.643	0.795393800282142\\
0.644	0.79605735970026\\
0.645	0.796720283409029\\
0.646	0.797382572821751\\
0.647	0.798044229346629\\
0.648	0.798705254386792\\
0.649	0.799365649340321\\
0.65	0.800025415600272\\
0.651	0.800684554554707\\
0.652	0.801343067586712\\
0.653	0.802000956074427\\
0.654	0.802658221391068\\
0.655	0.803314864904956\\
0.656	0.803970887979533\\
0.657	0.804626291973396\\
0.658	0.805281078240314\\
0.659	0.805935248129255\\
0.66	0.80658880298441\\
0.661	0.807241744145215\\
0.662	0.807894072946376\\
0.663	0.808545790717893\\
0.664	0.80919689878508\\
0.665	0.809847398468592\\
0.666	0.810497291084443\\
0.667	0.811146577944037\\
0.668	0.811795260354181\\
0.669	0.812443339617114\\
0.67	0.813090817030527\\
0.671	0.813737693887585\\
0.672	0.81438397147695\\
0.673	0.815029651082803\\
0.674	0.815674733984863\\
0.675	0.816319221458415\\
0.676	0.816963114774324\\
0.677	0.817606415199061\\
0.678	0.818249123994723\\
0.679	0.818891242419054\\
0.68	0.819532771725467\\
0.681	0.820173713163064\\
0.682	0.820814067976655\\
0.683	0.821453837406782\\
0.684	0.822093022689738\\
0.685	0.822731625057586\\
0.686	0.823369645738181\\
0.687	0.824007085955189\\
0.688	0.824643946928108\\
0.689	0.825280229872288\\
0.69	0.825915935998949\\
0.691	0.8265510665152\\
0.692	0.827185622624063\\
0.693	0.827819605524489\\
0.694	0.828453016411375\\
0.695	0.829085856475589\\
0.696	0.829718126903984\\
0.697	0.830349828879421\\
0.698	0.83098096358078\\
0.699	0.831611532182992\\
0.7	0.832241535857042\\
0.701	0.832870975770001\\
0.702	0.833499853085036\\
0.703	0.834128168961431\\
0.704	0.834755924554604\\
0.705	0.835383121016126\\
0.706	0.836009759493741\\
0.707	0.836635841131379\\
0.708	0.837261367069176\\
0.709	0.837886338443494\\
0.71	0.838510756386934\\
0.711	0.839134622028357\\
0.712	0.839757936492899\\
0.713	0.840380700901989\\
0.714	0.841002916373366\\
0.715	0.841624584021097\\
0.716	0.842245704955591\\
0.717	0.84286628028362\\
0.718	0.843486311108331\\
0.719	0.844105798529267\\
0.72	0.844724743642378\\
0.721	0.845343147540045\\
0.722	0.845961011311088\\
0.723	0.846578336040788\\
0.724	0.847195122810901\\
0.725	0.847811372699673\\
0.726	0.84842708678186\\
0.727	0.849042266128736\\
0.728	0.849656911808119\\
0.729	0.850271024884377\\
0.73	0.850884606418449\\
0.731	0.851497657467859\\
0.732	0.852110179086732\\
0.733	0.852722172325809\\
0.734	0.853333638232459\\
0.735	0.853944577850701\\
0.736	0.854554992221212\\
0.737	0.855164882381346\\
0.738	0.855774249365147\\
0.739	0.856383094203365\\
0.74	0.856991417923469\\
0.741	0.857599221549665\\
0.742	0.858206506102905\\
0.743	0.858813272600906\\
0.744	0.859419522058163\\
0.745	0.860025255485962\\
0.746	0.860630473892396\\
0.747	0.861235178282379\\
0.748	0.861839369657657\\
0.749	0.862443049016826\\
0.75	0.863046217355343\\
0.751	0.863648875665541\\
0.752	0.864251024936643\\
0.753	0.864852666154775\\
0.754	0.865453800302977\\
0.755	0.866054428361224\\
0.756	0.866654551306429\\
0.757	0.867254170112465\\
0.758	0.867853285750174\\
0.759	0.868451899187382\\
0.76	0.869050011388908\\
0.761	0.869647623316583\\
0.762	0.87024473592926\\
0.763	0.870841350182826\\
0.764	0.871437467030216\\
0.765	0.872033087421425\\
0.766	0.872628212303522\\
0.767	0.873222842620659\\
0.768	0.873816979314091\\
0.769	0.874410623322178\\
0.77	0.875003775580408\\
0.771	0.8755964370214\\
0.772	0.876188608574924\\
0.773	0.876780291167906\\
0.774	0.877371485724448\\
0.775	0.877962193165832\\
0.776	0.878552414410538\\
0.777	0.879142150374253\\
0.778	0.879731401969883\\
0.779	0.880320170107565\\
0.78	0.880908455694681\\
0.781	0.881496259635864\\
0.782	0.882083582833018\\
0.783	0.88267042618532\\
0.784	0.883256790589239\\
0.785	0.883842676938544\\
0.786	0.884428086124315\\
0.787	0.885013019034956\\
0.788	0.885597476556204\\
0.789	0.886181459571144\\
0.79	0.886764968960215\\
0.791	0.887348005601226\\
0.792	0.887930570369362\\
0.793	0.888512664137199\\
0.794	0.889094287774714\\
0.795	0.889675442149293\\
0.796	0.890256128125746\\
0.797	0.890836346566314\\
0.798	0.891416098330683\\
0.799	0.89199538427599\\
0.8	0.892574205256839\\
0.801	0.893152562125307\\
0.802	0.893730455730956\\
0.803	0.894307886920845\\
0.804	0.894884856539537\\
0.805	0.895461365429112\\
0.806	0.896037414429176\\
0.807	0.89661300437687\\
0.808	0.897188136106884\\
0.809	0.897762810451461\\
0.81	0.898337028240414\\
0.811	0.898910790301128\\
0.812	0.899484097458579\\
0.813	0.900056950535334\\
0.814	0.900629350351569\\
0.815	0.901201297725073\\
0.816	0.901772793471264\\
0.817	0.902343838403189\\
0.818	0.902914433331544\\
0.819	0.903484579064676\\
0.82	0.904054276408596\\
0.821	0.904623526166988\\
0.822	0.905192329141218\\
0.823	0.905760686130341\\
0.824	0.906328597931115\\
0.825	0.906896065338007\\
0.826	0.907463089143203\\
0.827	0.908029670136616\\
0.828	0.908595809105898\\
0.829	0.909161506836445\\
0.83	0.909726764111409\\
0.831	0.910291581711707\\
0.832	0.910855960416027\\
0.833	0.911419901000839\\
0.834	0.911983404240407\\
0.835	0.912546470906789\\
0.836	0.913109101769854\\
0.837	0.913671297597287\\
0.838	0.9142330591546\\
0.839	0.914794387205137\\
0.84	0.915355282510083\\
0.841	0.915915745828478\\
0.842	0.916475777917217\\
0.843	0.917035379531066\\
0.844	0.917594551422666\\
0.845	0.918153294342541\\
0.846	0.918711609039111\\
0.847	0.919269496258693\\
0.848	0.919826956745516\\
0.849	0.920383991241725\\
0.85	0.920940600487391\\
0.851	0.921496785220518\\
0.852	0.922052546177052\\
0.853	0.922607884090888\\
0.854	0.923162799693879\\
0.855	0.923717293715843\\
0.856	0.92427136688457\\
0.857	0.924825019925833\\
0.858	0.925378253563394\\
0.859	0.92593106851901\\
0.86	0.926483465512442\\
0.861	0.927035445261466\\
0.862	0.927587008481874\\
0.863	0.928138155887489\\
0.864	0.928688888190164\\
0.865	0.929239206099799\\
0.866	0.929789110324342\\
0.867	0.930338601569797\\
0.868	0.930887680540235\\
0.869	0.931436347937797\\
0.87	0.931984604462705\\
0.871	0.932532450813267\\
0.872	0.933079887685885\\
0.873	0.933626915775062\\
0.874	0.934173535773411\\
0.875	0.934719748371658\\
0.876	0.935265554258655\\
0.877	0.935810954121379\\
0.878	0.93635594864495\\
0.879	0.936900538512627\\
0.88	0.937444724405823\\
0.881	0.937988507004106\\
0.882	0.938531886985212\\
0.883	0.939074865025046\\
0.884	0.939617441797693\\
0.885	0.940159617975423\\
0.886	0.940701394228699\\
0.887	0.941242771226182\\
0.888	0.941783749634738\\
0.889	0.942324330119447\\
0.89	0.942864513343608\\
0.891	0.943404299968743\\
0.892	0.94394369065461\\
0.893	0.944482686059203\\
0.894	0.945021286838763\\
0.895	0.945559493647782\\
0.896	0.946097307139011\\
0.897	0.946634727963464\\
0.898	0.94717175677043\\
0.899	0.947708394207471\\
0.9	0.948244640920437\\
0.901	0.948780497553465\\
0.902	0.949315964748991\\
0.903	0.949851043147752\\
0.904	0.950385733388795\\
0.905	0.950920036109483\\
0.906	0.951453951945499\\
0.907	0.951987481530854\\
0.908	0.952520625497893\\
0.909	0.9530533844773\\
0.91	0.953585759098107\\
0.911	0.954117749987695\\
0.912	0.954649357771804\\
0.913	0.955180583074539\\
0.914	0.955711426518374\\
0.915	0.956241888724158\\
0.916	0.956771970311122\\
0.917	0.957301671896884\\
0.918	0.957830994097458\\
0.919	0.958359937527254\\
0.92	0.958888502799087\\
0.921	0.959416690524185\\
0.922	0.95994450131219\\
0.923	0.960471935771168\\
0.924	0.960998994507612\\
0.925	0.961525678126446\\
0.926	0.962051987231038\\
0.927	0.962577922423196\\
0.928	0.963103484303181\\
0.929	0.963628673469709\\
0.93	0.964153490519956\\
0.931	0.964677936049567\\
0.932	0.965202010652658\\
0.933	0.965725714921822\\
0.934	0.966249049448136\\
0.935	0.966772014821166\\
0.936	0.967294611628971\\
0.937	0.967816840458108\\
0.938	0.968338701893642\\
0.939	0.968860196519145\\
0.94	0.969381324916704\\
0.941	0.969902087666928\\
0.942	0.97042248534895\\
0.943	0.970942518540435\\
0.944	0.971462187817584\\
0.945	0.971981493755138\\
0.946	0.972500436926385\\
0.947	0.973019017903164\\
0.948	0.97353723725587\\
0.949	0.974055095553461\\
0.95	0.97457259336346\\
0.951	0.975089731251962\\
0.952	0.975606509783639\\
0.953	0.976122929521745\\
0.954	0.976638991028117\\
0.955	0.977154694863189\\
0.956	0.977670041585986\\
0.957	0.978185031754138\\
0.958	0.978699665923878\\
0.959	0.979213944650053\\
0.96	0.979727868486123\\
0.961	0.98024143798417\\
0.962	0.9807546536949\\
0.963	0.981267516167651\\
0.964	0.981780025950393\\
0.965	0.982292183589738\\
0.966	0.982803989630942\\
0.967	0.983315444617906\\
0.968	0.983826549093191\\
0.969	0.98433730359801\\
0.97	0.984847708672243\\
0.971	0.985357764854434\\
0.972	0.985867472681801\\
0.973	0.986376832690238\\
0.974	0.986885845414319\\
0.975	0.987394511387305\\
0.976	0.987902831141146\\
0.977	0.988410805206485\\
0.978	0.988918434112667\\
0.979	0.989425718387739\\
0.98	0.989932658558453\\
0.981	0.990439255150277\\
0.982	0.990945508687394\\
0.983	0.991451419692706\\
0.984	0.991956988687843\\
0.985	0.992462216193163\\
0.986	0.992967102727758\\
0.987	0.993471648809457\\
0.988	0.993975854954834\\
0.989	0.994479721679206\\
0.99	0.994983249496643\\
0.991	0.995486438919969\\
0.992	0.995989290460768\\
0.993	0.996491804629387\\
0.994	0.99699398193494\\
0.995	0.997495822885312\\
0.996	0.997997327987166\\
0.997	0.998498497745942\\
0.998	0.998999332665866\\
0.999	0.99949983324995\\
1	nan\\
};
\end{axis}

\begin{axis}[%
width=1.2in,
height=1.2in,
at={(0in,0in)},
scale only axis,
xmin=0,
xmax=1,
ymin=0,
ymax=1,
axis line style={draw=none},
ticks=none,
axis x line*=bottom,
axis y line*=left
]
\end{axis}
\end{tikzpicture}%

%% file: frameworktikz.tex
\tikzstyle{line} = [draw, -latex']
\begin{tikzpicture}[node distance = 1cm, auto]
\linespread{0.8}
\node[draw, rectangle, rounded corners,fill=white!80!black,  align=center] (numOfSamplePaths) { \scriptsize Number of sample paths};
\node[draw, rectangle, rounded corners, align=center, below=0.3cm of numOfSamplePaths] (concOfIndBer){\scriptsize Concentration of i.i.d. Bernoulli RVs (Lemma \ref{lemma:numberofpathswithsample})};
\node[draw, rectangle, rounded corners,fill=white!80!black,  align=center,  below=0.3cm of concOfIndBer] (numOfReachingPaths) {\scriptsize Lower bound on the number of sample paths \\ \scriptsize that has a sample from a given state-action pair};
\node[draw, rectangle, rounded corners, align=center, below=0.3cm of numOfReachingPaths] (concOfGeo){\scriptsize Concentration of i.i.d. Geometric RVs (Lemma \ref{lemma:numberofsamplesperstateaction})};
\node[draw, rectangle, rounded corners,fill=white!80!black,  align=center,  below=0.3cm of concOfGeo] (numOfSamples) {\scriptsize Lower bound on the number of samples\\ \scriptsize from a given state-action pair};
\node[draw, rectangle, rounded corners,fill=white!80!black,  align=center, left=4cm of numOfSamplePaths] (accVal) {\scriptsize Accuracy level for the value function};
\node[draw, rectangle, rounded corners,  align=center, below=0.3cm of accVal] (ourLemma){\scriptsize Lemma \ref{lemma:closemdpcloseestimate}};
\node[draw, rectangle, rounded corners,fill=white!80!black,  align=center, below=0.3cm of ourLemma] (accTran) {\scriptsize Sufficient accuracy for\\\scriptsize the transition probability estimates};
\node[draw, rectangle, rounded corners,  align=center, below=0.3cm of accTran] (concOfDepBer){\scriptsize Concentration of i.d. Bernoulli RVs (Lemma \ref{lemma:numberofsamplesfromsaq})};
\node[draw, rectangle, rounded corners, fill=white!80!black,  align=center, below=0.3cm of concOfDepBer] (reqSamples) {\scriptsize Sufficient number of samples\\\scriptsize for a given state-action pair};
\node[draw, rectangle, rounded corners,  align=center, below = 0.2cm of $(numOfSamples.south)!0.5!(reqSamples.south)$] (sufficientCond) {\scriptsize Sufficient condition};
\node[draw, rectangle, rounded corners, fill=white!80!black, align=center, below = 0.5cm of sufficientCond] (upperBoundSamples) {\scriptsize Upper bound on the required number of sample paths};
\path [line] (numOfSamplePaths) -- (concOfIndBer);
\path [line] (concOfIndBer) -- (numOfReachingPaths);
\path [line] (numOfReachingPaths) -- (concOfGeo);
\path [line] (concOfGeo) -- (numOfSamples);
\path [line] (accVal) -- (ourLemma);
\path [line] (ourLemma) -- (accTran);
\path [line] (accTran) -- (concOfDepBer);
\path [line] (concOfDepBer) -- (reqSamples);
\draw[-stealth] (numOfSamples) edge [bend right=10] (sufficientCond);
\draw[-stealth] (reqSamples) edge [bend left=10] (sufficientCond);
\path [line] (sufficientCond) -- (upperBoundSamples);
\end{tikzpicture}

%% file: beyondsamplemean.tex
\section{Beyond Sample Mean Estimators} \label{sec:dfestimator}
Sample mean estimators are preferred due to their asymptotic consistency and low variance. But, can we do better? The answer is positive given prior knowledge of the transition dynamics of the system. Reinforcement learning is often concerned with almost deterministic systems that have noisy dynamics, but the amount of noise is limited. For these systems, there exists an asymptotically consistent estimator that has a lower error than the sample mean estimator in the regime of low number of samples.

In our context, an MDP is almost deterministic if for all \(\genericstate \in \states \setminus \lbrace \deadstate \rbrace\) and \(\genericaction \in \actions\) there exists a \(\altstate \in \states \setminus \lbrace \deadstate \rbrace\) such that \(\probs(\genericstate, \genericaction, \altstate) \geq \processcontinueprob(1-\epsilon)\) with \(\epsilon \approx 0\).

Let \(X\) be a categorical random variable with support \(\lbrace 1, \ldots, K\rbrace\) and probability distribution \(\left[ p_{1}, \ldots, p_{K}\right]\). The sample mean estimator of \(p_{i}\) is \(\hat{p}_{i}^{SM} = \nicefrac{N_{i}}{N}\) where \(N\) is the number of i.i.d. samples and \(N_{i}\) is the number of samples with value \(i\). Let \(d = \arg\max_{i} N_{i}\). We define the deterministic-favored estimator as \[\hat{p}_{i}^{DF} = \frac{N_{i}+\sqrt{N}\mathds{1}_{d}(i)}{N + \sqrt{N}}\] where \(\mathds{1}_{d}(i)\) is \(1\) if \(i=d\) and 0 otherwise. We note that the deterministic-favored estimator does the opposite of the minimax estimator: while the minimax estimator favors a uniform posterior~\cite{wasserman2006all}, the deterministic-favored estimator boosts the estimated probability of the most frequent observation. 

The deterministic-favored estimator is asymptotically consistent. Furthermore, despite being biased, it has a lower mean squared error (MSE) than the sample mean estimator when \(X\) is almost deterministic.
\begin{theorem}\label{thm:mseofdfestimator}
The MSE of the deterministic-favored estimator \( \hat{p}_{i}^{DF} \) satisfies
\[\expectation{\left(\hat{p}_{i}^{DF} - p_{i}\right)^2} \leq \frac{N(1-p_{i})}{(N+\sqrt{N})^2} + \exp\left(- \frac{(2p_{i}-1)^2N}{12(1-p_{i})} \right)\] if \(p_{i} > \nicefrac{1}{2}.\)
\end{theorem}

The proof of Theorem \ref{thm:mseofdfestimator} is given in Appendix \ref{appendixdfestimator} of the supplementary material.

We note that the second term decays exponentially fast as \(N\) or \(p_{i}\) increases.
On the other hand, the MSE of the sample mean estimator is \(\nicefrac{Np_{i}(1-p_{i})}{N^{2}}\). The deterministic favored estimator performs better than the sample mean estimator if \(p_{i} \geq 1 - \epsilon \) where \(\epsilon \approx \nicefrac{N^2}{(N+\sqrt{N})^2}\). For instance, \(1-\nicefrac{N^2}{(N+\sqrt{N})^2} = 0.826\) for \(N = 100\). Symmetrically, for small \(p_{i}\), the deterministic-favored estimator performs better than the sample mean estimator if there exists \(j\) such that \(p_{j} \geq 1 - \epsilon \) where \(\epsilon \approx \nicefrac{N^2}{(N+\sqrt{N})^2}\). Overall, the deterministic-favored estimator outperforms the sample mean estimator for almost-deterministic random variables in the regime of low number of samples.

%% file: relatedwork.tex
\section{Related Work} \label{sec:relatedwork}
Off-policy evaluation literature mainly focuses on finding good estimators with a low bias and variance. The previous methods to solve this problem include importance sampling~\cite{liu2018breaking,xie2019towards}, variants of dynamic programming~\cite{hao2021bootstrapping,munos2008finite}, and model-based approaches~\cite{rashidinejad2021bridging,yan2022efficacy,mannor2004bias}. Model-based approaches are also used for offline policy optimization~\cite{yu2020mopo,yu2021combo}. 

\citet{yin2021near} studied the sample complexity of vanilla model-based off-policy evaluation and optimization in the finite-horizon setting under uniform coverage. Different from the infinite horizon setting, the sample transitions are i.i.d. in the finite-horizon setting.  For off-policy evaluation, \citet{yin2021near} showed a path sample complexity of \(\tilde{\mathcal{O}}\left(\nicefrac{H^{2}}{\min_{\substack{\genericstate \in \states \\  \genericaction \in \actions}} \occupancymeasure^{\behaviorpolicy} (\genericstate, \genericaction) \optimalitygap^{2} } + \nicefrac{H^{2}\sqrt{\cardstates \cardactions}}{\optimalitygap}\right)\) using martingale concentration inequalities where \(H\) is the time horizon. For the infinite-horizon setting, the bound that we give in Theorem \ref{thm:numberofrequiredpaths}, does not have any non-logarithmic dependencies in the number of actions and has a higher order dependency in the time horizon when \(\beta = 0\), i.e., uniform coverage. Similarly, for off-policy optimization \citet{yin2021near} showed a path sample complexity of \(\tilde{\mathcal{O}}\left(\nicefrac{H^{4}\cardstates}{\min_{\substack{\genericstate \in \states \\  \genericaction \in \actions}} \occupancymeasure^{\behaviorpolicy} (\genericstate, \genericaction) \optimalitygap^{2} } \right)\). The bound that we give in Corollary \ref{corollary:numberofrequiredpathsoptimization} matches the finite-horizon bound of \cite{yin2021near}. We also remark that our bound matches the sample complexity bound of vanilla asynchronous Q-learning~\cite{li2021q}.

The sample complexity of model-based off-policy optimization is studied extensively using pessimistic models~\cite{uehara2021pessimistic,ross2012agnostic,li2022settling,rashidinejad2021bridging}. Recently, \citet{li2022settling} showed a lower bound of \(\tilde{\mathcal{O}}\left( \nicefrac{\cardstates \max_{\substack{\genericstate \in \states \setminus \deadstate \\ \genericaction \in \actions}} \frac{\occupancymeasure^{\targetpolicy}(\genericstate, \genericaction)}{\occupancymeasure^{\behaviorpolicy}(\genericstate, \genericaction)} }{ (1-\processcontinueprob)^{3}\optimalitygap^{2}} \right)\) (independent) sample transitions and provided an algorithm with a matching sample complexity. The algorithm given in \cite{li2022settling} uses a pessimistic model, whereas we analyze the sample complexity of vanilla model-based off-policy optimization without pessimistic penalties. 

Different from the majority of model-based off-policy evaluation and optimization works, we consider that the samples are coming from time series data, i.e., paths, that obey the dynamics of the MDP. In the model-free setting, \citet{uehara2020minimax,yan2022efficacy,li2021q} consider that the samples come from an underlying Markov chain and showed that the independence assumption can be removed by assuming a mixing property for the underlying Markov chain. Unlike these works, we consider that the samples are (unbounded length) episodes of the MDP and do not require the underlying Markov chain to be ergodic. Consequently, we do not have the burn-in sampling costs due to the mixing time.

Existing literature on finite-horizon off-policy importance sampling~\cite{liu2018breaking,xie2019towards} shows that the variance of the importance sampling estimates grows exponentially with the horizon length. As a similar result, we show a sample complexity upper bound for the discounted infinite-horizon setting that has an exponential dependency in the expected time horizon.

To the best of our knowledge, previous model-based offline RL works use sample-mean estimators to estimate the transition probabilities. For the almost-deterministic environments, we propose an estimator that is motivated by the minimax estimator of the categorical random variables~\cite{wasserman2006all}.

%% file: conclusion.tex
\section{Conclusion}
We analyzed the sample complexity of vanilla model-based off-policy reinforcement learning with dependent samples. Despite its simple nature, the sample complexities of the vanilla model-based method are comparable to those of optimal algorithms. While the sample mean estimator becomes biased with dependent samples, our results show the order of the sample complexity remains the same compared to the case with independent samples. We also give an estimator that outperforms the sample mean estimator for almost-deterministic random variables. 

%% file: appendixnew2.tex
\newpage
\onecolumn

\begin{center}
    \textbf{\LARGE Supplementary Material}
\end{center}

 \section{Proof of Theorem \ref{thm:numberofrequiredpaths} and Corollary \ref{corollary:numberofrequiredpathsoptimization}}
 \label{appendix}

In this section, we explain the derivation of Theorem \ref{thm:numberofrequiredpaths} and Corollary \ref{corollary:numberofrequiredpathsoptimization}. We begin with the proofs of technical lemmas.

\begin{proof}[Proof of Lemma \ref{lemma:closemdpcloseestimate}]
We prove \(|\hat{\valuefunction{\targetpolicy}}(\initialstate) - \valuefunction{\targetpolicy}(\initialstate)| \leq \nicefrac{(\cardstates \cardactions)^{\beta}}{2(1-\processcontinueprob)^{1-\beta}} \) by showing that the flow difference under transition probability functions \(\probs\) and \(\hat{\probs}\) is bounded.

We define an MDP \(\overline{\mdp}=(\overline{\states}, \overline{\actions}, \overline{\probs}, \overline{\rewards}, \initialstate)\) that shares the same actions, reward function, and initial state with \(\mdp\).
\begin{itemize}
    \item The states of \(\overline{\mdp}\) is defined as \(\overline{\states} = \states \cup \Delta\) where \(\Delta= \lbrace \Delta_{(\genericstate, \genericaction)} | \genericstate \in \states \setminus \lbrace \deadstate \rbrace, \genericaction \in \actions\rbrace\).
    \item The actions of \(\overline{\mdp}\) is defined as \(\overline{\actions} = \actions \cup \lbrace \epsilon \rbrace\) where the available actions are the same with \(\mdp\) for the states in \(\states\), and the states in \(\lbrace \Delta_{(\genericstate, \genericaction)} | \genericstate \in \states \setminus \lbrace \deadstate \rbrace, \genericaction \in \actions\rbrace\) has a single action \(\epsilon\).
    
    \item The transition probability function of \(\overline{\mdp}\) is defined as \(\overline{\probs}(\genericstate, \genericaction, \altstate) = \min(\probs(\genericstate, \genericaction, \altstate), \hat{\probs}(\genericstate, \genericaction, \altstate))\) for all \(\genericstate \in \states \setminus \lbrace \deadstate \rbrace\), \(\genericaction \in \actions\), and \(\altstate \in \states\). In addition,  \(\overline{\probs}(\genericstate, \genericaction, \Delta_{(\genericstate, \genericaction)}) = \sum_{\altstate \in \states} | \probs(\genericstate, \genericaction, \altstate)- \hat{\probs}(\genericstate, \genericaction, \altstate) |/2\) for \(\genericstate \in \states \setminus \lbrace \deadstate \rbrace\) and \(\genericaction \in \actions\). Finally, \(\overline{\probs}(\Delta_{(\genericstate, \genericaction)}, \epsilon,\Delta_{(\genericstate, \genericaction)} ) = \processcontinueprob \) and \(\overline{\probs}(\Delta_{(\genericstate, \genericaction)}, \epsilon,\deadstate ) =1- \processcontinueprob \)  for all \(\Delta_{(\genericstate, \genericaction)} \in \Delta\).
    
    \item The reward function is defined such that \(\expectation{ \overline{\rewards}(\genericstate, \genericaction) } = \expectation{ \rewards(\genericstate, \genericaction) }\) if \(\genericstate \in \states\) and \(\genericaction \in \actions\), and \(\expectation{ \overline{\rewards}(\Delta_{(\genericstate, \genericaction)}, \epsilon) } = 1\) if \(\Delta_{(\genericstate, \genericaction)} \in \overline{\states} \setminus \states\) for all \(\Delta_{(\genericstate, \genericaction)} \in \Delta\).
\end{itemize}  In words, the state \(\Delta_{(\genericstate, \genericaction)}\) represents the state where the flow difference will be redirected if the transition probability functions \(\probs\) and \(\hat{\probs}\) disagree at \((\genericstate, \genericaction)\). Note that the flow difference at \((\genericstate, \genericaction)\) goes to a state \(\Delta_{(\genericstate, \genericaction)}\) where there is a single action \(\epsilon\) and the maximum possible reward is collected.

We define another MDP \(\underline{\mdp}=(\underline{\states}, \underline{\actions}, \underline{\probs}, \underline{\rewards}, \initialstate)\) that shares the same actions, reward function, and initial state with \(\mdp\). 
\begin{itemize}
    \item The states \(\underline{\states}\) of \(\underline{\mdp}\) is the same as \( \overline{\states}\).
    
    \item The transition probability function \(\underline{\probs}\) of \(\underline{\mdp}\) is the same as \(\overline{\probs}\).
    
    \item The reward function is defined such that \(\expectation{ \underline{\rewards}(\genericstate, \genericaction) } = \expectation{ \rewards(\genericstate, \genericaction) }\) if \(\genericstate \in \states\) and \(\genericaction \in \actions\), and \(\expectation{ \underline{\rewards}(\Delta_{(\genericstate, \genericaction)}, \epsilon) } = 0\) if \(\Delta_{(\genericstate, \genericaction)} \in \underline{\states} \setminus \states\) for all \(\Delta_{(\genericstate, \genericaction)} \in \Delta\).
\end{itemize}  Note that the only difference between \(\underline{\mdp}\) and \(\overline{\mdp}\) is the reward function: The flow difference at \((\genericstate, \genericaction)\) goes to a state \(\Delta_{(\genericstate, \genericaction)}\) where the minimum possible reward is collected.

Define an MDP \(\hat{\mdp}=(\states, \actions, \hat{\probs}, \rewards, \initialstate)\) that represents the estimated model.
 Let \begin{itemize}
     \item \(\valuefunction{\targetpolicy}_{\mdp}(\initialstate)\) be the value and  \(\occupancymeasure^{\targetpolicy}_{\mdp}(\genericstate, \genericaction)\) be the occupancy measure of \(\targetpolicy\) for \(\mdp\),
     \item \(\valuefunction{\targetpolicy}_{\hat{\mdp}}(\initialstate)\) be the value and  \(\occupancymeasure^{\targetpolicy}_{\hat{\mdp}}(\genericstate, \genericaction)\) be the occupancy measure of \(\targetpolicy\) for \(\hat{\mdp}\),
     \item \(\valuefunction{\targetpolicy}_{\overline{\mdp}}(\initialstate)\) be the value and  \(\occupancymeasure^{\targetpolicy}_{\overline{\mdp}}(\genericstate, \genericaction)\) be the occupancy measure of \(\targetpolicy\) for \(\overline{\mdp}\), and
     \item   \(\valuefunction{\targetpolicy}_{\underline{\mdp}}(\initialstate)\) be the value and  \(\occupancymeasure^{\targetpolicy}_{\underline{\mdp}}(\genericstate, \genericaction)\) be the occupancy measure of \(\targetpolicy\) for \(\underline{\mdp}\).
 \end{itemize}
 
 We note that \(\valuefunction{\targetpolicy}_{\mdp}(\initialstate) \leq \valuefunction{\targetpolicy}_{\overline{\mdp}}(\initialstate)\) and \(\valuefunction{\targetpolicy}_{\hat{\mdp}}(\initialstate) \leq \valuefunction{\targetpolicy}_{\overline{\mdp}}(\initialstate)\) since the redirected flow collects the maximum possible reward in \(\overline{\mdp}.\)  Similarly, note that \(\valuefunction{\targetpolicy}_{\mdp}(\initialstate) \geq \valuefunction{\targetpolicy}_{\underline{\mdp}}(\initialstate)\) and \(\valuefunction{\targetpolicy}_{\hat{\mdp}}(\initialstate) \geq \valuefunction{\targetpolicy}_{\underline{\mdp}}(\initialstate)\) since the redirected flow collects the maximum possible reward in \(\overline{\mdp}.\) Overall, we have \[\valuefunction{\targetpolicy}_{\underline{\mdp}}(\initialstate) \leq \valuefunction{\targetpolicy}_{\hat{\mdp}}(\initialstate), \valuefunction{\targetpolicy}_{\mdp}(\initialstate)  \leq  \valuefunction{\targetpolicy}_{\overline{\mdp}}(\initialstate). \]
 
 Now, we show that \(|\valuefunction{\targetpolicy}_{\underline{\mdp}}(\initialstate) - \valuefunction{\targetpolicy}_{\overline{\mdp}}(\initialstate)|\) is bounded. Note that \(\occupancymeasure^{\targetpolicy}_{\underline{\mdp}}(\genericstate, \genericaction) = \occupancymeasure^{\targetpolicy}_{\overline{\mdp}}(\genericstate, \genericaction)\) for all \(\states \setminus \lbrace \deadstate \rbrace\) and \(\genericaction \in \actions\). Also, \(\occupancymeasure^{\targetpolicy}_{\overline{\mdp}}(\genericstate, \genericaction) \leq \occupancymeasure^{\targetpolicy}_{\mdp}(\genericstate, \genericaction)\) for all \(\states \setminus \lbrace \deadstate \rbrace\) and \(\genericaction \in \actions\) since the flow that goes to \(\Delta\) never comes back to \(\states \setminus \lbrace \deadstate \rbrace\), and the remaining flow is always lower for \(\overline{\mdp}\). For the states in \(\Delta\), we have \(\occupancymeasure^{\targetpolicy}_{\overline{\mdp}}(\Delta_{(\genericstate, \genericaction)}, \epsilon) = \frac{1}{1-\processcontinueprob} \occupancymeasure^{\targetpolicy}_{\overline{\mdp}}(\genericstate, \genericaction) \sum_{\altstate \in \states}\overline{\probs}(\genericstate, \genericaction, \Delta_{(\genericstate, \genericaction)})\). The value functions have the following relationship
 \begin{align*}
      \valuefunction{\targetpolicy}_{\overline{\mdp}}(\initialstate)  &=\left( \sum_{\genericstate \in \states} \sum_{\genericaction \in \actions} \occupancymeasure^{\targetpolicy}_{\overline{\mdp}}(\genericstate, \genericaction) \expectation{ \rewards(\genericstate, \genericaction) } \right) + \left( \sum_{\Delta_{(\genericstate, \genericaction)} \in \Delta}  \occupancymeasure^{\targetpolicy}_{\overline{\mdp}}(\Delta_{(\genericstate, \genericaction)}, \epsilon) \expectation{ \overline{\rewards}(\Delta_{(\genericstate, \genericaction)}, \epsilon) } \right)
      \\
     &=       \valuefunction{\targetpolicy}_{\underline{\mdp}}(\initialstate)  + \left( \sum_{\Delta_{(\genericstate, \genericaction)} \in \Delta}  \occupancymeasure^{\targetpolicy}_{\overline{\mdp}}(\Delta_{(\genericstate, \genericaction)}, \epsilon) \expectation{ \overline{\rewards}(\Delta_{(\genericstate, \genericaction)}, \epsilon) } \right)
      \\
      &= \valuefunction{\targetpolicy}_{\underline{\mdp}}(\initialstate)+ \frac{1}{1-\processcontinueprob}\left( \sum_{\Delta_{(\genericstate, \genericaction)} \in \Delta}  \occupancymeasure^{\targetpolicy}_{\overline{\mdp}}(\genericstate, \genericaction) \sum_{\altstate \in \states}\overline{\probs}(\genericstate, \genericaction, \Delta_{(\genericstate, \genericaction)})\right)
            \\
      &\leq \valuefunction{\targetpolicy}_{\underline{\mdp}}(\initialstate)+ \frac{1}{1-\processcontinueprob} \left( \sum_{\Delta_{(\genericstate, \genericaction)} \in \Delta}  \occupancymeasure^{\targetpolicy}_{\mdp}(\genericstate, \genericaction) \sum_{\altstate \in \states}\overline{\probs}(\genericstate, \genericaction, \Delta_{(\genericstate, \genericaction)})\right)
      \\
      &\leq \valuefunction{\targetpolicy}_{\underline{\mdp}}(\initialstate)+ \frac{1}{1-\processcontinueprob} \left( \sum_{\Delta_{(\genericstate, \genericaction)} \in \Delta}  \occupancymeasure^{\targetpolicy}_{\mdp}(\genericstate, \genericaction) \frac{\alpha}{2 \occupancymeasure^{\targetpolicy}_{\mdp}(\genericstate, \genericaction)^{\beta}}\right)
      \\
      &= \valuefunction{\targetpolicy}_{\underline{\mdp}}(\initialstate)+  \frac{1}{1-\processcontinueprob} \sum_{\Delta_{(\genericstate, \genericaction)} \in \Delta}   \frac{\alpha \occupancymeasure^{\targetpolicy}_{\mdp}(\genericstate, \genericaction)^{1-\beta}}{2}
 \end{align*}
where the last inequality is due to \[\overline{\probs}(\genericstate, \genericaction, \Delta_{(\genericstate, \genericaction)}) = \frac{\sum_{\altstate \in \states} |\hat{\probs}(\genericstate, \genericaction, \altstate) - \probs(\genericstate, \genericaction, \altstate)|}{2} \leq \frac{\alpha}{2 \occupancymeasure^{\targetpolicy}(\genericstate, \genericaction)^{\beta}}.\] Since \(\sum_{\Delta_{(\genericstate, \genericaction)} \in \Delta} \occupancymeasure^{\targetpolicy}_{\mdp}(\genericstate, \genericaction) \leq \frac{1}{1-\processcontinueprob} \), \(|\Delta| = \cardstates \cardactions\), and \(x^{1-\beta}\) is a concave function of \(x\) for \(0 \leq \beta \leq 1\), we have \[ \valuefunction{\targetpolicy}_{\underline{\mdp}}(\initialstate)+  \sum_{\Delta_{(\genericstate, \genericaction)} \in \Delta}   \frac{\alpha \occupancymeasure^{\targetpolicy}_{\mdp}(\genericstate, \genericaction)^{1-\beta}}{2 (1-\processcontinueprob)} \leq \valuefunction{\targetpolicy}_{\underline{\mdp}}(\initialstate) + \sum_{\Delta_{(\genericstate, \genericaction)} \in \Delta}   \frac{\alpha \left(\frac{1}{\cardstates \cardactions (1-\processcontinueprob)}\right)^{1-\beta}}{2(1-\processcontinueprob)} = \valuefunction{\targetpolicy}_{\underline{\mdp}}(\initialstate) + \frac{\alpha (\cardstates \cardactions)^{\beta}}{2(1-\processcontinueprob)^{2-\beta}}.\]

Since \[\valuefunction{\targetpolicy}_{\underline{\mdp}}(\initialstate) \leq \valuefunction{\targetpolicy}_{\hat{\mdp}}(\initialstate), \valuefunction{\targetpolicy}_{\mdp}(\initialstate)  \leq  \valuefunction{\targetpolicy}_{\overline{\mdp}}(\initialstate)\] and \[|\valuefunction{\targetpolicy}_{\underline{\mdp}}(\initialstate) - \valuefunction{\targetpolicy}_{\overline{\mdp}}(\initialstate)| \leq   \frac{\alpha (\cardstates \cardactions)^{\beta}}{2(1-\processcontinueprob)^{2-\beta}},\] we have \[|\valuefunction{\targetpolicy}_{\mdp}(\initialstate) - \valuefunction{\targetpolicy}_{\hat{\mdp}}(\initialstate)| \leq   \frac{\alpha (\cardstates \cardactions)^{\beta}}{2(1-\processcontinueprob)^{2-\beta}}.\]

\end{proof}

\newpage

\begin{proof}[Proof of Lemma \ref{lemma:numberofsamplesfromsaq}] 
We first define some notation to derive a under concentration bound that works with a random stopping time. We note that \[\hat{\probs}(\genericstate, \genericaction, \altstate) =  \frac{\processcontinueprob \numberofsampletransitions(\genericstate, \genericaction, \altstate)}{\sum_{\altstate \in \states\setminus \lbrace \deadstate \rbrace }\numberofsampletransitions(\genericstate, \genericaction, \altstate)}  \] for all $\altstate \in \states \setminus \lbrace \deadstate \rbrace$ assuming that $\sum_{\altstate \in \states\setminus \lbrace \deadstate \rbrace }\numberofsampletransitions(\genericstate, \genericaction, \altstate) \geq 1$.  Let $\hat{R}(\genericstate, \genericaction, \altstate) = \hat{\probs}(\genericstate, \genericaction, \altstate)/ \processcontinueprob$. 

Let $m \geq 1$ be a fixed integer. Let $\numberofsampletransitions_{m}(\genericstate, \genericaction, \altstate)$ denote the number of sample transitions from $\genericstate$ to $\altstate$ under action $\genericaction$ within the first $m$ sample transitions from $\genericstate$ to $\states\setminus \lbrace \deadstate \rbrace$ under action $\genericaction$. We define \[\hat{\probs}_{m}(\genericstate, \genericaction, \altstate) =  \frac{\processcontinueprob \numberofsampletransitions_{m}(\genericstate, \genericaction, \altstate)}{m}.  \] Similar to the relationship between $\hat{\probs}(\genericstate, \genericaction, \altstate)$ and $\hat{R}(\genericstate, \genericaction, \altstate)$, we define $\hat{R}_{m}(\genericstate, \genericaction, \altstate) = \hat{\probs}_{m}(\genericstate, \genericaction, \altstate)/ \processcontinueprob$.

By Theorem 2.1 of \cite{weissman2003inequalities}, we have 
\begin{equation}
    \Pr\left(\sum_{\altstate \in \states \setminus \lbrace \deadstate \rbrace} |\hat{R}_{m}(\genericstate, \genericaction, \altstate) - \probs(\genericstate, \genericaction, \altstate)/\processcontinueprob| \geq c  \right) \leq 2^{\cardstates} \exp(-m c^2/2) \label{ineq:empricall1bound}
\end{equation}
 for all \(c > 0\). Let \(\mu = \sum_{\altstate \in \states\setminus \lbrace \deadstate \rbrace }\numberofsampletransitions(\genericstate, \genericaction, \altstate)\) be the random stopping time with respect to the filtration generated by the previous sample transitions. Note that this stopping time is potentially dependent on the previous transitions from all states. By letting \(c = \sqrt{\frac{2}{\mu}\left( 2 \log(\mu) + \log\left(\nicefrac{2^\cardstates 5}{3 \failureprob'}\right) \right)}\), we get 

\begin{align}
    &\Pr\left(\sum_{\altstate \in \states \setminus \lbrace \deadstate \rbrace} |\hat{R}_{\mu}(\genericstate, \genericaction, \altstate) - \probs(\genericstate, \genericaction, \altstate)/\processcontinueprob| \geq \sqrt{\frac{2}{\mu}\left( 2 \log(\mu) + \log\left(\nicefrac{2^\cardstates 5 }{3 \failureprob'}\right) \right)}  \right)
    \\
    &\leq \sum_{m = 1}^{\infty} \Pr\left(  \sum_{\altstate \in \states \setminus \lbrace \deadstate \rbrace} |\hat{R}_{m}(\genericstate, \genericaction, \altstate) - \probs(\genericstate, \genericaction, \altstate)/\processcontinueprob| \geq \sqrt{\frac{2}{m}\left( 2 \log(m) + \log\left(\nicefrac{2^\cardstates 5 }{3 \failureprob'}  \right)  \right)} \right) \label{ineq:unionbound}
     \\
    &\leq \sum_{m = 1}^{\infty}  2^{\cardstates} \exp\left( -\frac{m}{2} \sqrt{\frac{2}{m}\left( 2 \log(m) + \log\left(\nicefrac{2^\cardstates 5}{3 \failureprob'}  \right)  \right)}^2 \right)  = \leq \sum_{m = 1}^{\infty}  \frac{\failureprob'}{m^2} = \frac{\pi^2 \failureprob'}{10} \leq  \failureprob' \label{ineq:l1boundapplied}
\end{align}
where \eqref{ineq:unionbound} is due to the union bound and \eqref{ineq:l1boundapplied} is due to \eqref{ineq:empricall1bound}.

If \[\mu = \sum_{\altstate \in \states\setminus \lbrace \deadstate \rbrace }\numberofsampletransitions(\genericstate, \genericaction, \altstate) \geq \frac{40 \cardstates}{(\optimalitygap')^2} \log\left(\frac{1}{\optimalitygap'}\right) \log\left( \frac{5}{3 \failureprob'} \right), \] we have 

\begin{align}
    &\sqrt{\frac{2}{\mu}\left( 2 \log(\mu) + \log\left(\nicefrac{2^\cardstates 5 }{3 \failureprob'}\right) \right)} 
    \\
    &= \sqrt{\frac{4 \log\left( 40 \cardstates (\nicefrac{1}{\optimalitygap'})^2 \log(\nicefrac{1}{\optimalitygap'}) \log(\nicefrac{5}{3 \failureprob'})  \right)}{  30 \cardstates (\nicefrac{1}{\optimalitygap'})^2 \log(\nicefrac{1}{\optimalitygap'}) \log(\nicefrac{5}{3 \failureprob'}) } + \frac{\log(\nicefrac{2^\cardstates 5 }{3 \failureprob'})}{  40 \cardstates (\nicefrac{1}{\optimalitygap'})^2 \log(\nicefrac{1}{\optimalitygap'}) \log(\nicefrac{5}{3 \failureprob'}) }}
    \\
    &= \optimalitygap' \left( \frac{4 \log (40 \cardstates)}{40 \cardstates \log(\nicefrac{1}{\optimalitygap'}) \log(\nicefrac{5}{3 \failureprob'})} + \frac{8 \log (\nicefrac{1}{\optimalitygap'})}{40 \cardstates \log(\nicefrac{1}{\optimalitygap'}) \log(\nicefrac{5}{3 \failureprob'})} + \frac{4 \log(\log (\nicefrac{1}{\optimalitygap'}))}{40 \cardstates \log(\nicefrac{1}{\optimalitygap'}) \log(\nicefrac{5}{3 \failureprob'})}  \right. \nonumber
    \\
    & \left.  + \frac{4 \log(\log (\nicefrac{5}{3 \failureprob'}))}{40 \cardstates \log(\nicefrac{1}{\optimalitygap'}) \log(\nicefrac{5}{3 \failureprob'})} +  \frac{\log(\nicefrac{2^\cardstates 5 }{3 \failureprob'})}{40 \cardstates \log(\nicefrac{1}{\optimalitygap'}) \log(\nicefrac{5}{3 \failureprob'})} \right)^{1/2}
    \\
    &\leq \optimalitygap' \sqrt{\frac{4 \log (40 \cardstates)}{40 \cardstates} + \frac{8 \log (\nicefrac{1}{\optimalitygap'})}{40 \cardstates \log(\nicefrac{1}{\optimalitygap'})} + \frac{4 \log(\log (\nicefrac{1}{\optimalitygap'}))}{40 \cardstates \log(\nicefrac{1}{\optimalitygap'})}  + \frac{4 \log(\log (\nicefrac{5}{3 \failureprob'}))}{40 \cardstates \log(\nicefrac{5}{3 \failureprob'})} +  \frac{\log(2) \cardstates \log(\nicefrac{ 5 }{3 \failureprob'})}{40 \cardstates \log(\nicefrac{5}{3 \failureprob'})} } \label{ineq:somearelargerthan1}
    \\
    &\leq \optimalitygap' \sqrt{\frac{18}{80} + \frac{8 }{80 } + \frac{ 4}{80}  + \frac{4 }{80} +  \frac{2 }{40} }  \leq \optimalitygap' \label{ineq:Sislargerthan2}
\end{align}
where \eqref{ineq:somearelargerthan1} is due to \(\log(\nicefrac{1}{\optimalitygap'}) \geq 1\) and \(\log(\nicefrac{5}{3 \failureprob'}) \geq 1\), and \eqref{ineq:Sislargerthan2} is due to \(\cardstates \geq 2\) and \(x \geq \log(x)\). Consequently, if \[\mu = \sum_{\altstate \in \states\setminus \lbrace \deadstate \rbrace }\numberofsampletransitions(\genericstate, \genericaction, \altstate) \geq \frac{40 \cardstates}{(\optimalitygap')^2} \log\left(\frac{1}{\optimalitygap'}\right) \log\left( \frac{5}{3 \failureprob'} \right), \] we have 
\[\Pr\left(\sum_{\altstate \in \states \setminus \lbrace \deadstate \rbrace} |\hat{R}_{\mu}(\genericstate, \genericaction, \altstate) - \probs(\genericstate, \genericaction, \altstate)/\processcontinueprob| \geq \optimalitygap'  \right) \leq \failureprob'.\] 

Note that \(|\hat{R}_{\mu}(\genericstate, \genericaction, \altstate) - \probs(\genericstate, \genericaction, \altstate)/\processcontinueprob| = | \hat{\probs}_{\mu}(\genericstate, \genericaction, \altstate) - \probs(\genericstate, \genericaction, \altstate)| / \processcontinueprob\) for all \(\altstate \in \states \setminus \lbrace \deadstate \rbrace\), and  \(| \hat{\probs}_{\mu}(\genericstate, \genericaction, \deadstate) - \probs(\genericstate, \genericaction, \deadstate)|  = 0\) by definition. Hence, if \[\mu = \sum_{\altstate \in \states\setminus \lbrace \deadstate \rbrace }\numberofsampletransitions(\genericstate, \genericaction, \altstate) \geq \frac{40 \cardstates}{(\optimalitygap')^2} \log\left(\frac{1}{\optimalitygap'}\right) \log\left( \frac{5}{3 \failureprob'} \right), \] we have 
\[\Pr\left(\sum_{\altstate \in \states} |\hat{\probs}_{\mu}(\genericstate, \genericaction, \altstate) - \probs(\genericstate, \genericaction, \altstate)| \geq \processcontinueprob \optimalitygap'  \right) \leq \failureprob'.\]

\end{proof}

\newpage

\begin{proof}[Proof of Lemma \ref{lemma:numberofpathswithsample}]
A sample path has at least 1 sample from \(\genericstate, \genericaction\) with probability \(\reachprob^{\behaviorpolicy}(\genericstate, \genericaction)\). Every path that has at least 1 sample from \(\genericstate, \genericaction\), has a sample transition from \(\genericstate\) to \(\genericstate \setminus \lbrace \deadstate \rbrace\) under action \(\genericaction\) with probability \(\processcontinueprob\). Overall, a sample path has at least 1 sample from \(\genericstate\) to \(\genericstate \setminus \lbrace \deadstate \rbrace\) under action \(\genericaction\) with probability \(\processcontinueprob \reachprob^{\behaviorpolicy}(\genericstate, \genericaction)\) independent from the other sample paths. 

Let \(X^{i}\) be a Bernoulli random variable that is \(1\) if \(i\)-th sample path has at least 1 sample from \(\genericstate\) to \(\genericstate \setminus \lbrace \deadstate \rbrace\) under action \(\genericaction\), and 0 otherwise. We have 
\[\Pr\left(\sum_{i=1}^{\numberofpaths} X^{i} \leq (1-\varepsilon') \mu \right) \leq \exp(-\nicefrac{\mu (\varepsilon')^2}{3})\] by the Chernoff-Hoeffding bound where \(\mu = \numberofpaths \processcontinueprob  \reachprob^{\behaviorpolicy}(\genericstate, \genericaction)\) and \(\varepsilon' \in (0,1)\). Let \(\varepsilon' = 1 - \nicefrac{\numberofpaths'}{\numberofpaths \processcontinueprob  \reachprob^{\behaviorpolicy}(\genericstate, \genericaction)}\).

Let \[\numberofpaths \geq \frac{6}{\processcontinueprob \reachprob^{\behaviorpolicy}(\genericstate, \genericaction)} \max \left(  \numberofpaths',  \log(\nicefrac{1}{\failureprob'})\right).\] If \( \numberofpaths' \geq \log(\nicefrac{1}{\failureprob'})\), then \(\varepsilon' \geq \nicefrac{5}{6} \), \(\mu \geq 6N'\), and 
\[\Pr\left(\sum_{i=1}^{\numberofpaths} X^{i} \leq \numberofpaths'  \right) \leq \exp\left(-\frac{150 \numberofpaths' }{108}\right) \leq \exp\left(-\frac{150 \log(\nicefrac{1}{\failureprob'}) }{108}\right) \leq \failureprob'.\] 
If \( \log(\nicefrac{1}{\failureprob'}) < \numberofpaths' \), then \(\varepsilon' \geq \nicefrac{5}{6} \), \(\mu\geq 6 \log(\nicefrac{1}{\failureprob'})\), and \[\Pr\left(\sum_{i=1}^{\numberofpaths} X^{i} \leq \numberofpaths'  \right) \leq \exp\left(-\frac{6 \log(\nicefrac{1}{\failureprob'}) (\varepsilon')^2 }{3}\right) \leq \exp\left(-\frac{150 \log(\nicefrac{1}{\failureprob'}) }{108}\right) \leq \failureprob'.\] 

\end{proof}

\newpage

\begin{proof}[Proof of Lemma \ref{lemma:numberofsamplesperstateaction}]

We first note that the number of samples from a state-action pair over a path is a geometric random variable with parameter \(1-\loopprob^{\behaviorpolicy}(\genericstate, \genericaction)\) given that there is a sample transition from that state-action pair in the path. Let \(\numberofsampletransitions^{i}(\genericstate, \genericaction, \altstate)\) denote the number of sample transitions from \(\genericstate\) to \(\altstate\) under action \(\genericaction\) in the \(i\)-th sample path that has a transition from \(\genericstate\) to \(\altstate\) under action \(\genericaction\). Formally, 
\[ \Pr \left(\sum_{\altstate \in \states }\numberofsampletransitions^{i}(\genericstate, \genericaction, \altstate) = m \Bigg| \sum_{\altstate \in \states }\numberofsampletransitions^{i}(\genericstate, \genericaction, \altstate) \geq 1 \right) = (1-\loopprob^{\behaviorpolicy}(\genericstate, \genericaction))^{m-1} \loopprob^{\behaviorpolicy}(\genericstate, \genericaction).   \] This is because of the stationarity of \(\behaviorpolicy\) and Markovianity of the MDP.

Similarly, we have \begin{equation} \label{bound:geometricreturn}
     \Pr \left(\sum_{\altstate \in \states \setminus \lbrace \deadstate \rbrace }\numberofsampletransitions^{i}(\genericstate, \genericaction, \altstate) = m \Bigg| \sum_{\altstate \in \states \setminus \lbrace \deadstate \rbrace }\numberofsampletransitions^{i}(\genericstate, \genericaction, \altstate) \geq 1 \right) = (1-\loopprob^{\behaviorpolicy}(\genericstate, \genericaction))^{m-1} \loopprob^{\behaviorpolicy}(\genericstate, \genericaction).
\end{equation} This is because of two facts: 
\begin{enumerate}
    \item The probability of transitioning to a state in \(\states \setminus \lbrace \deadstate \rbrace\) given that the current state is \(\genericstate\) and current action is \(\genericaction\), is \(\processcontinueprob\),
    \item The probability of taking action \(\genericaction\) at state \(\genericstate\) again under stationary policy \(\policy\) given that the current state is \(\genericstate\), current action is \(\genericaction\), and next state is in \(\states \setminus \lbrace \deadstate \rbrace\), is \(\nicefrac{\loopprob^{\behaviorpolicy}(\genericstate, \genericaction)}{\processcontinueprob}.\)
\end{enumerate}

Note that \( \sum_{i=1}^{\numberofpaths'} \sum_{\altstate \in \states \setminus \lbrace \deadstate \rbrace }\numberofsampletransitions^{i}(\genericstate, \genericaction, \altstate)\) is a sum of \(\numberofpaths'\) geometric random variables with parameter \(1-\loopprob^{\behaviorpolicy}(\genericstate, \genericaction)\). By Theorem 3.1 of \cite{janson2018tail}, we have 
\[\Pr \left( \sum_{i=1}^{\numberofpaths'} \sum_{\altstate \in \states \setminus \lbrace \deadstate \rbrace }\numberofsampletransitions^{i}(\genericstate, \genericaction, \altstate) \leq k \right) \leq \exp\left(-\numberofpaths' \left(-1 + \log\left( \frac{\numberofpaths'}{k (1- \loopprob^{\behaviorpolicy})}\right)\right)\right). \] Let \(\numberofpaths' = \max \left( 8 k (1-\loopprob^{\behaviorpolicy}(\genericstate, \genericaction)), \log(\nicefrac{1}{\failureprob'}) \right).\) If \(8 k (1-\loopprob^{\behaviorpolicy}(\genericstate, \genericaction)) \geq \log(\nicefrac{1}{\failureprob'})\), then \[\Pr \left( \sum_{i=1}^{\numberofpaths'} \sum_{\altstate \in \states \setminus \lbrace \deadstate \rbrace }\numberofsampletransitions^{i}(\genericstate, \genericaction, \altstate) \leq k \right) \leq \exp\left(-8 k (1-\loopprob^{\behaviorpolicy}(\genericstate, \genericaction)) \right) \leq \failureprob'. \]If \(8 k (1-\loopprob^{\behaviorpolicy}(\genericstate, \genericaction)) < \log(\nicefrac{1}{\failureprob'})\), then \[\Pr \left( \sum_{i=1}^{\numberofpaths'} \sum_{\altstate \in \states \setminus \lbrace \deadstate \rbrace }\numberofsampletransitions^{i}(\genericstate, \genericaction, \altstate) \leq k \right) \leq \exp\left(-\log(\nicefrac{1}{\failureprob'} )  \left(-1 + \log\left( \frac{\numberofpaths'}{k (1- \loopprob^{\behaviorpolicy})}\right)\right)\right) \leq  \failureprob'. \]

\end{proof}

\newpage

\begin{proof}[Proof of Theorem \ref{thm:numberofrequiredpaths}]
By Lemma \ref{lemma:closemdpcloseestimate}, if \[\sum_{\altstate \in \states} |\hat{\probs}(\genericstate, \genericaction, \altstate) - \probs(\genericstate, \genericaction, \altstate)| \leq \frac{\optimalitygap (1-\processcontinueprob)^{2-\beta}}{ \occupancymeasure^{\targetpolicy}(\genericstate, \genericaction)^{\beta} (\cardstates\cardactions)^{\beta}}\] for all \(\genericstate \in \states\) and \(\genericaction \in \actions\), then \[|\valuefunction{\targetpolicy}_{\mdp}(\initialstate) - \valuefunction{\targetpolicy}_{\hat{\mdp}}(\initialstate)| \leq \optimalitygap .\] The estimate for state \(\deadstate\) already satisfies \(\sum_{\altstate \in \states} |\hat{\probs}(\deadstate, \genericaction, \altstate) - \probs(\deadstate, \genericaction, \altstate)| \leq \frac{\optimalitygap (1-\processcontinueprob)^{2-\beta}}{ \occupancymeasure^{\targetpolicy}(\genericstate, \genericaction)^{\beta} (\cardstates\cardactions)^{\beta}}\). 
Define \[D = \left\lbrace (\genericstate, \genericaction) \bigg|  \occupancymeasure^{\targetpolicy}(\genericstate, \genericaction) \geq \frac{(\nicefrac{\optimalitygap}{2})^{\nicefrac{1}{\beta}} (1-\processcontinueprob)^{\nicefrac{(2-\beta)}{\beta}}}{\cardstates \cardactions} \right\rbrace.\] We note that every \((\genericstate, \genericaction) \not \in D\), already satisfies \(\sum_{\altstate \in \states} |\hat{\probs}(\genericstate, \genericaction, \altstate) - \probs(\genericstate, \genericaction, \altstate)| \leq \frac{\optimalitygap (1-\processcontinueprob)^{2-\beta}}{ \occupancymeasure^{\targetpolicy}(\genericstate, \genericaction)^{\beta} (\cardstates\cardactions)^{\beta}}.\)

By Lemma \ref{lemma:numberofsamplesfromsaq} and the facts that \(\occupancymeasure^{\targetpolicy}(\genericstate, \genericaction) \leq \nicefrac{1}{(1-\processcontinueprob)}\) and \(\beta \leq 1\), \[n_{\genericstate, \genericaction}:=\frac{40 \cardstates^{\beta + 1} \cardactions^{\beta}\processcontinueprob^{2}\occupancymeasure^{\targetpolicy}(\genericstate, \genericaction)^{2\beta}}{(1-\processcontinueprob)^{4-2\beta}\optimalitygap^2} \log\left(\frac{ \processcontinueprob \cardstates \cardactions}{(1-\processcontinueprob)^{3}\optimalitygap}\right) \log\left( \frac{5\cardstates\cardactions}{ \failureprob} \right)\] samples from every \((\genericstate, \genericaction) \in D\) to \(\states \setminus \lbrace \deadstate \rbrace\) is enough to achieve \(\frac{\optimalitygap (1-\processcontinueprob)^{2-\beta}}{ \occupancymeasure^{\targetpolicy}(\genericstate, \genericaction)^{\beta} (\cardstates\cardactions)^{\beta}}\) accuracy for every \(\genericstate \in \states \setminus \lbrace \deadstate \rbrace\) and \(\genericaction \in \actions\) with probability \(1- \nicefrac{\failureprob}{3}\) by the union bound.

Let \(N'\) be the number of sample paths that has a sample from state \(\genericstate\) to \(\states \setminus \lbrace \deadstate \rbrace\) under action \(\genericaction\). By Lemma \ref{lemma:numberofsamplesperstateaction}, if \[N' \geq \max \left( 8 n_{\genericstate, \genericaction}  (1-\loopprob^{\behaviorpolicy}(\genericstate, \genericaction)) , \log\left(\frac{3\cardstates \cardactions}{\failureprob}\right) \right), \] then the number of samples from \(\genericstate \) and \(\genericaction \) to \(\states \setminus \lbrace \deadstate \rbrace\) is at least \(n_{\genericstate,\genericaction}\) with probability at least 1-\(\nicefrac{\failureprob}{3 \cardstates \cardactions}\). By Lemma \ref{lemma:numberofpathswithsample} and the union bound, if the number of sample paths satisfy
\begin{align}
    \numberofpaths &\geq \frac{6}{\processcontinueprob \reachprob^{\behaviorpolicy}(\genericstate, \genericaction)} \max \left(  8 n_{\genericstate, \genericaction}  (1-\loopprob^{\behaviorpolicy}(\genericstate, \genericaction)) , \log\left(\frac{3\cardstates \cardactions}{\failureprob}\right) \right)
    \\
    &=\max   \left(  \frac{48 n_{\genericstate, \genericaction}}{\processcontinueprob \occupancymeasure^{\behaviorpolicy}(\genericstate, \genericaction)}, \frac{6}{\processcontinueprob \reachprob^{\behaviorpolicy}(\genericstate, \genericaction)} \log\left(\frac{3\cardstates \cardactions}{\failureprob}\right) \right),
\end{align} then the number of samples from \(\genericstate \) and \(\genericaction \) to \(\states \setminus \lbrace \deadstate \rbrace\) is at least \(n_{\genericstate, \genericaction}\) with probability at least 1-\(\nicefrac{2\failureprob}{3 \cardstates \cardactions}\). By the union bound, if \[\numberofpaths \geq \underset{(\genericstate, \genericaction) \in D}{\max} \max   \left(  \frac{1920 \cardstates^{\beta + 1} \cardactions^{\beta} \processcontinueprob}{(1-\processcontinueprob)^{4-2\beta}\optimalitygap^2} \cdot \frac{\occupancymeasure^{\targetpolicy}(\genericstate, \genericaction)^{2\beta}}{\occupancymeasure^{\behaviorpolicy}(\genericstate, \genericaction)} \log\left(\frac{ \processcontinueprob \cardstates \cardactions}{(1-\processcontinueprob)^{3}\optimalitygap}\right) \log\left( \frac{5\cardstates\cardactions}{ \failureprob} \right), \frac{6}{\processcontinueprob \reachprob^{\behaviorpolicy}(\genericstate, \genericaction)} \log\left(\frac{3\cardstates \cardactions}{\failureprob}\right) \right),\] then with probability at least 1-\(\nicefrac{2\failureprob}{3 \cardstates \cardactions}\), the number of samples from every \((\genericstate, \genericaction) \in D\) to \(\states \setminus \lbrace \deadstate \rbrace\) is at least \(n_{\genericstate, \genericaction}\). 
Consequently, by the union bound, \[|\valuefunction{\targetpolicy}_{\mdp}(\initialstate) - \valuefunction{\targetpolicy}_{\hat{\mdp}}(\initialstate)| \leq \optimalitygap \] with probability at least \(1-\failureprob\).
\end{proof}

\newpage

\begin{proof}[Proof of Corollary \ref{corollary:numberofrequiredpathsoptimization}]
By noting that \(\occupancymeasure^{\policy^{*}}(\genericstate, \genericaction), \occupancymeasure^{\policy'}(\genericstate, \genericaction) \leq \nicefrac{1}{(1-\processcontinueprob)}\) and Theorem \ref{thm:numberofrequiredpaths}, if 
\[\numberofpaths \geq \underset{\substack{\genericstate \in \states' \\ \genericaction \in \actions}}{\max} \max   \left(  \frac{7680 \cardstates\processcontinueprob}{(1-\processcontinueprob)^{5}\optimalitygap^2} \cdot \frac{1}{\occupancymeasure^{\behaviorpolicy}(\genericstate, \genericaction)} \log\left(\frac{ \processcontinueprob}{(1-\processcontinueprob)^{3}\optimalitygap}\right) \log\left( \frac{10\cardstates\cardactions}{ \failureprob} \right), \frac{6}{\processcontinueprob \reachprob^{\behaviorpolicy}(\genericstate, \genericaction)} \log\left(\frac{6\cardstates \cardactions}{\failureprob}\right) \right),\] then \[|\valuefunction{\policy^{*}}_{\mdp}(\initialstate) - \valuefunction{\policy^{*}}_{\hat{\mdp}}(\initialstate)| \leq \nicefrac{\optimalitygap}{2} \] and \[|\valuefunction{\policy'}_{\mdp}(\initialstate) - \valuefunction{\policy'}_{\hat{\mdp}}(\initialstate)| \leq \nicefrac{\optimalitygap}{2} \] with probability at least \(1-\failureprob\). Since \(\valuefunction{\policy'}_{\hat{\mdp}}(\initialstate) \geq \valuefunction{\policy^{*}}_{\hat{\mdp}}(\initialstate)\) due to the optimality of \(\policy'\) against \(\policy^{*}\) on \(\hat{\mdp}\), \[|\valuefunction{\policy^{*}}_{\mdp}(\initialstate) - \valuefunction{\policy'}_{\mdp}(\initialstate)| \leq \optimalitygap \] with probability at least \(1-\failureprob\).

\end{proof}

%% file: appendiximportance.tex
\section{Notes for Off-Policy Estimation with Vanilla Importance Sampling} \label{appendiximportance}
Let \(\pathdist{\target}\) and \(\pathdist{\behavior}\) denote the distribution of paths under the target and behavior policies, respectively. The quantity \(\mathbb{E}_{\apath\sim\pathdist{\targetpolicy}}\left[l(\apath)\right]\) is equal to the KL divergence \(KL(\pathdist{\target}||\pathdist{\behavior} )\) between the path distributions \(\pathdist{\target}\) and \(\pathdist{\behavior}\). When \(\targetpolicy\) and \(\behaviorpolicy\) are stationary, we have \[KL(\pathdist{\target}||\pathdist{\behavior} ) = \sum_{\genericstate \in \states \setminus \lbrace \deadstate \rbrace} \sum_{\genericaction \in \actions} \occupancymeasure^{\targetpolicy}(\genericstate, \genericaction) \log\left(\frac{{\targetpolicy}(\genericstate, \genericaction)}{{\behaviorpolicy}(\genericstate, \genericaction)}\right).\] By replacing \(\frac{{\targetpolicy}(\genericstate, \genericaction)}{{\behaviorpolicy}(\genericstate, \genericaction)}\) with \(\max_{\substack{\genericstate \in \states \setminus \deadstate \\ \genericaction \in \actions}} \frac{\targetpolicy(\genericstate, \genericaction)}{\behaviorpolicy(\genericstate, \genericaction)}\) and noting that \(\sum_{\genericstate \in \states \setminus \lbrace \deadstate \rbrace} \sum_{\genericaction \in \actions} \occupancymeasure^{\targetpolicy}(\genericstate, \genericaction)\), we get \[\mathbb{E}_{\apath\sim\pathdist{\targetpolicy}}\left[l(\apath)\right] = KL(\pathdist{\target}||\pathdist{\behavior} )  \leq \frac{1}{1-\processcontinueprob} \log\left(\max_{\substack{\genericstate \in \states \setminus \deadstate \\ \genericaction \in \actions}} \frac{\targetpolicy(\genericstate, \genericaction)}{\behaviorpolicy(\genericstate, \genericaction)} \right).\]

We have 
\[\mathrm{Std}_{\apath\sim\pathdist{\targetpolicy}}(l(\apath)) = \sqrt{\mathbb{E}_{\apath\sim\pathdist{\targetpolicy}}\left[l(\apath)^{2}\right] - \mathbb{E}_{\apath\sim\pathdist{\targetpolicy}}\left[l(\apath)\right]^{2}} \leq \sqrt{\mathbb{E}_{\apath\sim\pathdist{\targetpolicy}}\left[l(\apath)^{2}\right]}.\] Let \(\tau(\apath)=\min \lbrace t| \apath= \genericstate_{0}\genericaction_{0} \genericstate_{1} \genericaction_{1}\ldots,  \genericstate_{t} = \deadstate\), i.e., the time index that \(\deadstate\) is reached. Note that every path reaches \(\deadstate\) with probability \(1-\processcontinueprob\) at every time step independently. Consequently, \(\tau(\apath)\) with \(\apath\sim\pathdist{\targetpolicy}\) has geometrical distribution with parameter \(\processcontinueprob\), i.e., \(\Pr(\tau(\apath) = i| \apath\sim\pathdist{\targetpolicy}) = (1-\processcontinueprob)\processcontinueprob^{i-1}\) for all \(i\geq 1\). We have 
\begin{align}
    \mathbb{E}_{\apath\sim\pathdist{\targetpolicy}}\left[l(\apath)^{2}\right] &= \sum_{\apath=\genericstate_{0}\genericaction_{0} \genericstate_{1} \genericaction_{1} \ldots\in Support(\pathdist{\target})}  \Pr(\apath|\targetpolicy) \log^{2}\left(\frac{ \Pr(\apath|\targetpolicy)}{ \Pr(\apath|\behaviorpolicy)}\right)
    \\
    &= \sum_{\apath=\genericstate_{0}\genericaction_{0} \genericstate_{1} \genericaction_{1} \ldots\in Support(\pathdist{\target})}  \Pr(\apath|\targetpolicy) \left( \sum_{t=0}^{\tau(\apath)-1} \log\left(\frac{ \targetpolicy(\genericstate_{t}, \genericaction_{t})}{ \behaviorpolicy(\genericstate_{t}, \genericaction_{t})}\right)\right)^{2}
    \\
    &\leq  \sum_{\apath=\genericstate_{0}\genericaction_{0} \genericstate_{1} \genericaction_{1} \ldots\in Support(\pathdist{\target})}  \Pr(\apath|\targetpolicy) \left( \sum_{t=0}^{\tau(\apath)-1} \log\left(\max_{\substack{\genericstate \in \states \setminus \deadstate \\ \genericaction \in \actions}} \frac{\targetpolicy(\genericstate, \genericaction)}{\behaviorpolicy(\genericstate, \genericaction)} \right) \right)^{2}
    \\
    &=  \sum_{\apath=\genericstate_{0}\genericaction_{0} \genericstate_{1} \genericaction_{1} \ldots\in Support(\pathdist{\target})}  \Pr(\apath|\targetpolicy)  \tau(\apath)^{2} \log^{2}\left(\max_{\substack{\genericstate \in \states \setminus \deadstate \\ \genericaction \in \actions}} \frac{\targetpolicy(\genericstate, \genericaction)}{\behaviorpolicy(\genericstate, \genericaction)} \right) 
    \\
    &=\log^{2}\left(\max_{\substack{\genericstate \in \states \setminus \deadstate \\ \genericaction \in \actions}} \frac{\targetpolicy(\genericstate, \genericaction)}{\behaviorpolicy(\genericstate, \genericaction)} \right) \sum_{\apath=\genericstate_{0}\genericaction_{0} \genericstate_{1} \genericaction_{1} \ldots\in Support(\pathdist{\target})}  \Pr(\apath|\targetpolicy)  \tau(\apath)^{2}
    \\
    &=\log^{2}\left(\max_{\substack{\genericstate \in \states \setminus \deadstate \\ \genericaction \in \actions}} \frac{\targetpolicy(\genericstate, \genericaction)}{\behaviorpolicy(\genericstate, \genericaction)} \right) \sum_{t=1}^{\infty} (1-\processcontinueprob)\processcontinueprob^{i-1} i^2
    \\
    &=\log^{2}\left(\max_{\substack{\genericstate \in \states \setminus \deadstate \\ \genericaction \in \actions}} \frac{\targetpolicy(\genericstate, \genericaction)}{\behaviorpolicy(\genericstate, \genericaction)} \right) \frac{1+\gamma}{(1-\gamma)^2}.
\end{align}
By noting that \(\gamma \leq 1\), we get
\[\mathrm{Std}_{\apath\sim\pathdist{\targetpolicy}}(l(\apath)) \leq \frac{\sqrt{2}}{1-\processcontinueprob} \log\left(\max_{\substack{\genericstate \in \states \setminus \deadstate \\ \genericaction \in \actions}} \frac{\targetpolicy(\genericstate, \genericaction)}{\behaviorpolicy(\genericstate, \genericaction)}\right). \]

%% file: appendixdfestimator.tex
 \section{Proof of Theorem \ref{thm:mseofdfestimator}}
 \label{appendixdfestimator}
 
 \begin{proof}[Proof of Theorem \ref{thm:mseofdfestimator}]
Let \(E\) denote the event that \(i = d\). The MSE of \(\hat{p}_{i}^{DF}\) is 
\begin{equation} \label{eqn:mseofdfestimator}
    \expectation{\left(\hat{p}_{i}^{DF} - p_{i}\right)^2} = \Pr(\neg E)\expectation{\left(\hat{p}_{i}^{DF} - p_{i}\right)^2 | \neg E} + \Pr(E)\expectation{\left(\hat{p}_{i}^{DF} - p_{i}\right)^2 | E}.
\end{equation}

Also, let \(\hat{p}_{i}^{iF}\) be an estimator that always favors \(i\) such that \[\hat{p}_{i}^{iF} = \frac{N_{i}+\sqrt{N}}{N + \sqrt{N}}.\] The MSE of \(\hat{p}_{i}^{iF}\) is 
\begin{subequations} \label{eqn:mseofifestimator}
\begin{align}
        \expectation{\left(\hat{p}_{i}^{iF} - p_{i}\right)^2} &= \mathrm{Var}\left[ \frac{N_{i} + \sqrt{N}}{N+\sqrt{N}} \right] + \left( \expectation{\frac{N_{i} + \sqrt{N}}{N+\sqrt{N}}} - p_{i} \right)^{2} 
        \\
        &=\frac{N p_{i} (1-p_{i})}{(N+ \sqrt{N})^2} + \left( \frac{N p_{i} + \sqrt{N}}{N+\sqrt{N}} - p_{i} \right)^{2}
        \\
        &=\frac{N (1-p_{i})}{N+\sqrt{N}}.
\end{align}
\end{subequations}
Also, note that 
\begin{subequations} 
\begin{align}
    \expectation{\left(\hat{p}_{i}^{iF} - p_{i}\right)^2} &= \Pr(\neg E)\expectation{\left(\hat{p}_{i}^{iF} - p_{i}\right)^2 | \neg E} + \Pr(E)\expectation{\left(\hat{p}_{i}^{iF} - p_{i}\right)^2 | E}
    \\
    &\geq \Pr(E)\expectation{\left(\hat{p}_{i}^{iF} - p_{i}\right)^2 | E}
        \\
    &= \Pr(E)\expectation{\left(\hat{p}_{i}^{DF} - p_{i}\right)^2 | E} \label{ineq:mseofif}
\end{align}
\end{subequations} where the last step is due to the fact that \(\hat{p}_{i}^{iF} = \hat{p}_{i}^{DF}\) if \(E\) happens.

We first bound \(\Pr(\neg E)\expectation{\left(\hat{p}_{i}^{DF} - p_{i}\right)^2 | \neg E}\) in \eqref{eqn:mseofdfestimator}. Note that \(0 \leq \hat{p}_{i}^{DF}\leq 1\).  Let \(F\) denote the event that \(N_{i} \geq \nicefrac{N}{2}.\) Note that \(\Pr(\neg E) \leq \Pr(\neg F)\) since every time \(F\) happens \(E\) must happen. By the the Chernoff-Hoeffding bound, we get \[\Pr(\neg F) \leq \exp\left(- \frac{(2p_{i}-1)^2N}{12(1-p_{i})} \right).\] Since \(0 \leq \hat{p}_{i}^{DF}\leq 1\), we have \[\expectation{\left(\hat{p}_{i}^{DF} - p_{i}\right)^2 | \neg E} \leq 1.\] Hence, we have \[\Pr(\neg E)\expectation{\left(\hat{p}_{i}^{DF} - p_{i}\right)^2 | \neg E} \leq \exp\left(- \frac{(2p_{i}-1)^2N}{12(1-p_{i})} \right).\]

Combining this with \eqref{ineq:mseofif} in \eqref{eqn:mseofdfestimator}, we get the desired result. 

\end{proof}